\documentclass[a4paper,11pt]{article}
\usepackage{fullpage} 

\usepackage{amsmath}
\usepackage{amssymb}
\usepackage{amsthm}
\usepackage{mathtools}
\usepackage{mathdots}
\usepackage{appendix}
\usepackage{graphicx}
\usepackage{enumerate}
\usepackage{microtype}
\usepackage{mleftright}
\usepackage{mdframed}
\usepackage{authblk}
\usepackage{tcolorbox}
\usepackage[english]{babel}
\usepackage[utf8]{inputenc}
\usepackage{algorithm}
\usepackage[noend]{algpseudocode}
\usepackage{xspace}
\usepackage{tikz}
\usepackage{pgfplots}
\pgfplotsset{width=8cm,compat=1.9}
\usepackage{mathdots}
\usepackage{url}
\usepackage{forest}
\forestset{default preamble={for tree={s sep+=1cm}}}
\usepackage{caption,subcaption}
\usepackage[noadjust]{cite}
\usepackage{hyperref}
\usepackage{lipsum}

\tcbset{
    rounded corners,
    colback = white,
    before skip = 0.2cm,
    after skip = 0.5cm,
    boxrule = 1.5pt,
    arc = 5pt
}

\makeatletter
\renewcommand{\ALG@name}{Protocol}
\makeatother
\allowdisplaybreaks

\graphicspath{ {./figures/} }

\usepackage{algorithm}
\usepackage[noend]{algpseudocode}

\newtheorem{theorem}{Theorem}[section]
\newtheorem*{theorem*}{Theorem}
\newtheorem{proposition}[theorem]{Proposition}
\newtheorem{lemma}[theorem]{Lemma}
\newtheorem{corollary}[theorem]{Corollary}
\newtheorem*{corollary*}{Corollary}

\newtheorem{observation}[theorem]{Observation}

\newtheorem{exmp}[theorem]{Example}

\theoremstyle{definition}
\newtheorem{definition}[theorem]{Definition}

\newcommand{\cU}{\mathcal{U}}
\newcommand{\cH}{\mathcal{H}}
\newcommand{\cX}{\mathcal{X}}
\newcommand{\cY}{\mathcal{Y}}
\newcommand{\cS}{\mathcal{S}}

\newcommand{\cW}{\mathcal{W}}

\newcommand{\Lrn}{\mathsf{Lrn}}

\newcommand{\Agn}{\mathsf{Agn}}

\newcommand{\RSOA}{\mathsf{RandSOA}}
\newcommand{\SOA}{\mathsf{SOA}}

\newcommand{\depth}{\mathsf{depth}}

\newcommand{\LD}{\mathtt{L}}
\newcommand{\RL}{\mathtt{RL}}
\newcommand{\M}{\texttt{M}}

\newcommand{\BRSOA}{\mathsf{BoundedRandSOA}}
\newcommand{\WRSOA}{\mathsf{WeightedRandSOA}}
\newcommand{\WSOA}{\mathsf{WeightedSOA}}
\newcommand{\BW}{\mathsf{BW}}
\newcommand{\FTL}{\mathsf{FTL}}
\newcommand{\Squint}{\mathsf{Squint}}

\newcommand{\rounds}{\mathbf{T}}
\newcommand{\expertrounds}{\mathbf{T}}

\DeclareMathOperator*{\argmin}{arg\,min}
\DeclareMathOperator*{\argmax}{arg\,max}

\newcommand\blfootnote[1]{%
  \begingroup
  \renewcommand\thefootnote{}\footnote{#1}%
  \addtocounter{footnote}{-1}%
  \endgroup
}

\newif\ifstoc
\stocfalse
\newif\iffull
\fullfalse

\newif\ifsicomp
\sicompfalse

\DeclareMathOperator{\up}{up}
\DeclareMathOperator{\Low}{Low}

\begin{document}
\title{Optimal Prediction Using Expert Advice\\ and\\ Randomized Littlestone Dimension}

\ifstoc
\author{Anonymous}
\else
\author[1,2]{Yuval Filmus} 
\author[3]{Steve Hanneke}
\author[1]{Idan Mehalel}
\author[2,1,4]{Shay Moran}
\affil[1]{The Henry and Marilyn Taub Faculty of Computer Science, Technion, Israel}
\affil[2]{Faculty of Mathematics, Technion, Israel}
\affil[3]{Department of Computer Science, Purdue University, USA}
\affil[4]{Google Research, Israel}
\fi

\maketitle

\begin{abstract}
A classical result in online learning characterizes the optimal mistake bound achievable by deterministic learners using the Littlestone dimension (Littlestone '88).
We prove an analogous result for randomized learners: we show that the optimal \emph{expected} mistake bound in learning a class $\cH$ equals its \emph{randomized Littlestone dimension}, which we define as follows: it is the largest $d$ for which there exists a tree shattered by $\cH$ whose \emph{average} depth is $2d$.
We further study optimal mistake bounds in the agnostic case, as a function of the number of mistakes made by the best function in $\cH$, denoted by $k$. Towards this end we introduce the $k$-Littlestone dimension and its randomized variant, and use them to characterize the optimal deterministic and randomized mistake bounds.
Quantitatively, we show that the optimal randomized mistake bound for learning a class with Littlestone dimension $d$ is $k + \Theta (\sqrt{k d} + d )$ (equivalently, the optimal regret is $\Theta(\sqrt{kd} + d)$). This also implies an optimal deterministic mistake bound of $2k + \Theta(d) + O(\sqrt{k d})$, thus resolving an open question which was studied by Auer and Long ['99].
 
As an application of our theory, we revisit the classical problem of prediction using expert advice: 
about 30 years ago Cesa-Bianchi, Freund, Haussler, Helmbold, Schapire and Warmuth studied prediction using expert advice, 
provided that the best among the $n$ experts makes at most $k$ mistakes, and asked what are the optimal mistake bounds (as a function of~$n$ and $k$). Cesa-Bianchi, Freund, Helmbold, and Warmuth ['93, '96] provided a nearly optimal bound for deterministic learners, and left the randomized case as an open problem. We resolve this question by providing an optimal learning rule in the randomized case, and showing that its expected mistake bound equals half of the deterministic bound of Cesa-Bianchi et al.\ ['93, '96], up to negligible additive terms. 
In contrast with previous works by Abernethy, Langford, and Warmuth [’06], and by Br\^anzei and Peres [’19],
our result applies to all pairs $n,k$, and does so via a unified analysis using the randomized Littlestone dimension.

In our proofs we develop and use optimal learning rules, which can be seen as natural variants of the Standard Optimal Algorithm ($\SOA$) of Littlestone: a weighted variant in the agnostic case, and a probabilistic variant in the randomized case. We conclude the paper with suggested directions for future research and open questions.

\end{abstract}

\pagebreak

\setcounter{tocdepth}{2}
\tableofcontents

\pagebreak

\section{Introduction}

A recurring phenomenon in learning theory is that different notions of learnability are captured by combinatorial parameters.
    Notable examples include the Vapnik--Chervonenkis (VC) dimension which characterizes PAC learnability~\cite{VapnikC74,blumer1989learnability}
    and the Littlestone dimension which characterizes online learnability~\cite{littlestone1988learning,bendavid2009agnostic}.
    Other examples include the Daniely--Shalev-Shwartz and Natarajan dimensions in multiclass PAC learning~\cite{natarajan1989learning,daniely2014optimal,brukhim2022characterization}, the star number, disagreement coefficient, and inference dimension in interactive learning~\cite{Hanneke2014dis,HannekeY15active,Kane17interactive},
    the statistical query dimension in learning with statistical queries~\cite{Feldman17},
    the representation dimension, one-way communication complexity, and Littlestone dimension in differentially private learning~\cite{FeldmanX15,Beimel19Pure,AlonBLMM22},
    and others.
    


One of the simplest and most appealing characterizations is that of online learnability by the Littlestone dimension.
    In his seminal work, Nick Littlestone proved that the optimal mistake-bound in online learning a class $\cH$ 
    is \emph{exactly} the Littlestone dimension of $\cH$ \cite{littlestone1988learning}.
    Thus, not only does the Littlestone dimension qualitatively captures online learnability,
    it also provides an exact quantitative characterization of the best possible mistake bound.
    This distinguishes the Littlestone dimension from other dimensions in learning theory, 
    which typically only provide asymptotic bounds on the learning complexity.

However, the exact quantitative characterization of the optimal mistake bound by the Littlestone dimension 
    applies only in the noiseless \emph{realizable} setting and only for \emph{deterministic} learners.
    In particular, it does not apply in the more general and well-studied setting of \emph{agnostic} online learning. 
    The reason it does not apply is twofold: (i) because the agnostic setting allows for non-realizable sequences, 
    and (ii) because randomized learners are in fact necessary.\footnote{Randomized learners are necessary in the following sense: any agnostic online learner for a class $\cH$ must be randomized, provided that $\cH$ contains at least two functions~\cite{cover65}, see also~\cite[Chapter 21.2]{shalev2014understanding}.}
    This suggests the following question, which guides this work:
    
\smallskip
\begin{tcolorbox}
\begin{center}
    Is there a natural dimension which captures the optimal expected mistake bound in learning a class $\cH$ using randomized learners? How about the agnostic setting when there is no $h\in \cH$ which is consistent with input data? 
\end{center}
\end{tcolorbox}

The main contribution of this work formalizes and proves affirmative answers to these questions.
\iffull
\else
Some of the technical material is omitted from this manuscript and can be found in the full version which is accessible online in~\cite{FullVersion}.
\fi


\paragraph{Organization.} In the next section we present the main results of this work.
Then, in Section~\ref{sec:intro-technical-overview} we provide a short technical overview, 
where we outline the main ideas we use in our proofs.
The remaining sections contain the complete proofs.

\section{Main results}
\label{sec:intro-main-results}

This section assumes familiarity with standard definitions and terminology from online learning.
    We refer the unfamiliar reader to Section~\ref{sec:prelim}, which introduces the online learning model
    and related basic definitions in a self-contained manner.

\subsection{Realizable Case}
In his seminal work from 1988, Nick Littlestone studied the optimal mistake bound in online learning an hypothesis class $\cH$ 
by deterministic learning rules in the realizable setting~\cite{littlestone1988learning};
that is, under the assumption that the input data sequence is consistent with a function $h\in\cH$.

\paragraph{Littlestone dimension.}
Let $\cX$ be the domain, and let $\cH$ be a class of ``$\cX\to\{0,1\}$'' predictors.
    The Littlestone dimension of $\cH$, denoted $\LD(\cH)$, is the maximal depth 
    of a binary complete decision tree~$T$ which is shattered by~$\cH$.
    That is, a decision tree $T$ whose nodes are associated with points from $\cX$
    and whose edges are associated with labels from $\{0,1\}$ such that each 
    of the branches (root-to-leaf paths) in $T$ is realized by some $h\in\cH$.

Littlestone proved that the optimal mistake bound achievable by deterministic learners equals the Littlestone dimension:
\begin{theorem}[Deterministic Mistake Bound~\cite{littlestone1988learning}]\label{thm:LittlestoneIntro}
The optimal \underline{deterministic} mistake bound in online learning $\cH$ in the realizable setting is equal to its Littlestone dimension,~$\LD(\cH)$.
\end{theorem}
Littlestone further described a natural deterministic learning rule, which he dubbed the \emph{Standard Optimal Algorithm} ($\SOA$),
that makes at most $\LD(\cH)$ mistakes on every realizable input sequence.

\paragraph{Randomized Littlestone dimension.}
Our first main result shows that a natural probabilistic variant of the Littlestone dimension characterizes
    the optimal expected mistake bound for randomized learners.

\begin{tcolorbox}
\begin{definition}[Randomized Littlestone Dimension] 
Let $T$ be binary tree, and consider a random walk on $T$ 
    that starts at the root, goes to the left or right child with probability~$1/2$, 
    and continues recursively in the same manner until reaching a leaf.
    Let $E_T$ denote the expected length of a random branch which is produced by this process.

The \emph{randomized Littlestone dimension} of a class~$\cH$, denoted by $\RL(\cH)$, is defined by
\[
\RL(\cH) =\frac{1}{2} \sup_{T \text{ shattered}} E_T.
\]
\end{definition}
\end{tcolorbox}

To compare the randomized Littlestone dimension with the Littlestone dimension,  
    notice that the Littlestone dimension is equal to $\sup\,\{m_T : T \text{ shattered}\}$,
    where $m_T$ is the minimum length of a branch in $T$.
    Thus, the difference is that in $\RL(\cH)$ we take the expected depth rather than the minimal depth,
    and multiply by a factor of $1/2$.\footnote{From a learning theoretic perspective it is easy to see that $\RL(\cH)\leq \LD(\cH)$, because randomized learners are more general than deterministic ones.
    Interestingly, this inequality is less obvious from a combinatorial perspective:
    indeed, for every fixed tree $T$ we have that $E_T \geq m_T$ 
    (because the expected length of a branch is at least the minimal length),
    but it is not a~priori clear why the inequality is reversed when $E_T$ is replaced by $E_T/2$
    and we take supremum over all shattered trees.
    }

\begin{tcolorbox}
\begin{theorem}[Main Result (i): Randomized Mistake Bound]\label{thm:characterizationIntro}
The optimal \underline{randomized} mistake bound in online learning $\cH$ in the realizable setting is equal to its randomized Littlestone dimension,~$\RL(\cH)$.
\end{theorem}
\end{tcolorbox}
We also provide an optimal randomized learning rule which can be seen as 
    a probabilistic adaptation of Littlestone's classical $\SOA$ algorithm.
    See Section~\ref{sec:technical-oveview-dimension} for a brief overview, 
    and Section~\ref{sec:thm:characterization} for the proof.

The connection between online learning problems and random walks was identified in the online learning literature \cite{abernethy2008optimal2, luo2014towards, gravin2016towards}. \cite{abernethy2008optimal2} asked, conceptually, how general is this connection. Our results show that it is indeed quite general, in the sense that it yields the optimal algorithm for every hypothesis class.

\subsection{Agnostic Case}
We next consider the agnostic setting, in which we no longer assume that the input sequence of examples is consistent with $\cH$.
    Our second main result characterizes the optimal expected mistake bound in this setting.

A common approach for handling the agnostic case is to assume a \emph{bounded horizon} and analyze the \emph{regret}.
    That is, it is assumed that the length of the input sequence (called the \emph{horizon})  
    is a given parameter $\rounds\in\mathbb{N}$, and the goal is to design learning rules
    whose mistake bound is competitive with that of the best $h\in \cH$ up to an additive term which is negligible
    in~$\rounds$ (this term is called the \emph{regret} of the algorithm). 

The bounded horizon assumption simplifies the design of learning rules, by allowing them to depend on $\rounds$.
    A notable example is the celebrated \emph{Multiplicative Weights} (MW) learning rule,
    whose learning rate depends on $\rounds$.
    This assumption can then be lifted by standard \emph{doubling tricks}.\footnote{E.g.\ start by running the algorithm with $\rounds=2$, and double $\rounds$ when reaching the $(\rounds+1)$'st example.}

\paragraph{The \texorpdfstring{$k$}{k}-realizable setting.}    
In this work we consider an alternative approach: instead of assuming a bound $\rounds$ on the horizon, 
    we assume a bound $k$ on the number of mistakes made by the best function in the class.
    Notice that this assumption can also be lifted by suitable doubling tricks as we demonstrate in Section~\ref{sec:variations},
    where we also extend our results to the bounded-horizon setting.
    
The upshot of this approach is that it allows for a precise combinatorial characterization 
    of the optimal mistake bound via a natural generalization of the Littlestone dimension.
    
\subsubsection{$k$-Littlestone Dimension}
Let $\cH$ be an hypothesis class, and let $k \in \mathbb{N}$. 
    A sequence of examples $S = \{(x_i,y_i)\}_{i=1}^t$ is \emph{$k$-realizable} by $\cH$ if there exists $h \in \cH$ such that $h(x_i) \neq y_i$ for at most $k$ indices $i$. 
    In the $k$-realizable setting we assume that the input sequence given to the learner is $k$-realizable. 
    Notice that the case $k=0$ amounts to realizability by $\cH$.
    We say that a decision tree is \emph{$k$-shattered by $\cH$} if every branch is $k$-realizable by $\cH$.
    The corresponding deterministic and randomized $k$-Littlestone dimensions of a class $\cH$ are
\[
\LD_k(\cH) = \sup_{T  \text{ } k\text{-shattered}} m_T
\quad \text{and} \quad
\RL_k(\cH) =\frac{1}{2} \sup_{T  \text{ } k\text{-shattered}} E_T.
\]

\begin{tcolorbox}
\begin{theorem}[Main Result (ii): $k$-Littlestone Dimension]\label{thm:characterization_kIntro}
Let $\cH$ be an hypothesis class.
\begin{enumerate}
    \item The optimal deterministic mistake bound in online learning $\cH$ in the $k$-realizable setting equals its $k$-Littlestone dimension,~$\LD_k(\cH)$.
    \item The optimal randomized mistake bound in online learning $\cH$ in the $k$-realizable setting equals its $k$-randomized Littlestone dimension, $\RL_k(\cH)$.
\end{enumerate}
\end{theorem}
\end{tcolorbox}

We also provide optimal learning rules which can be seen as 
    weighted variants of Littlestone's classical $\SOA$ algorithm. 
    See Section~\ref{sec:technical-oveview-dimension} for a brief overview and Section~\ref{sec:k-realizable} for the proof.
    

As a consequence of this perspective, we prove the following theorem which provides tight regret bounds in terms of the Littlestone dimension.

\begin{tcolorbox}
\begin{theorem}[{Main Result (iii): Optimal Regret Bounds for Littlestone Classes}]
\label{thm:k-realizable-littlestone-bound-Intro}
{Let $\cH$ be an hypothesis class and let $k\in\mathbb{N}$. Then
\[
\RL_k(\cH) = k + \Theta\!\left( \sqrt{ k \cdot \LD(\cH) } + \LD(\cH) \right).\\
\]
In particular, the optimal regret in online learning $\cH$ is $\Theta\!\left( \sqrt{ k \cdot \LD(\cH) } + \LD(\cH) \right)$,
where $k$ is the number of mistakes made by the best function in $\cH$.}
\end{theorem}
\end{tcolorbox}

This improves and refines over results by \cite{auer1999structural,alon2021adversarial}.
The work by~\cite{alon2021adversarial} determined an optimal regret bound of $\Theta\mleft(\LD(\cH) + \sqrt{T\cdot \LD(\cH)}\mright)$, where $T$ is the time horizon. The above bound refines it by replacing $T$ with $k\leq T$.
The work by~\cite{auer1999structural} studies the optimal deterministic mistake bound in online learning $\cH$ in the $k$-realizable setting. Theorem~\ref{thm:k-realizable-littlestone-bound-Intro} and the deterministic lower bound of \cite{littlestone1994weighted} imply that the deterministic mistake bound is
\begin{align*}
\LD_k(\cH) = 2k +  \Theta\!\left( \LD(\cH) \right) + O\!\left( \sqrt{ k \cdot \LD(\cH) } \right). 
\end{align*}

This improves over \cite[Theorem 4.4]{auer1999structural}, which gives an upper bound of $(2+2.5\epsilon)k + O\mleft(\frac{1}{\epsilon} \log \frac{1}{\epsilon}\mright) \LD(\cH)$ for every $0 < \epsilon \leq 1/20$. See Section~\ref{sec:k-realizable-littlestone-bound} for the proof of Theorem~\ref{thm:k-realizable-littlestone-bound-Intro}.

\subsection{Prediction Using Expert Advice (Results)} \label{sec:expert-intro}

In this section, we consider the problem of \emph{prediction using expert advice}~\cite{vovk1990aggregating,littlestone1994weighted}. 
This problem studies a repeated guessing game between a learner and an adversary. 
In each round of the game, the learner needs to guess the label that the adversary chooses.
In order to do so, the learner can use the advice of $n$ experts. 
Formally, each round $i$ in the game proceeds as follows:

\begin{enumerate}[(i)]
\item The experts present predictions $\hat{y}_i^{(1)},\ldots,\hat{y}_i^{(n)} \in \{0,1\}$.
\item The learner predicts a value $p_i \in [0,1]$.
\item The adversary reveals the true answer $y_i \in \{0,1\}$, and the learner suffers the loss $|y_i - p_i|$.
\end{enumerate}
The value $p_i$ should be understood as the probability (over the learner's randomness) of predicting~$y_i=1$. 
    Notice that the adversary only gets to see $p_i$, which reflects the assumption that the adversary does not know the learner's internal randomness. 
    Notice also that the suffered loss $\lvert y_i - p_i\rvert$ exactly captures the probability that the learner makes a mistake.
    The above is a standard way to model randomized learners in online learning, see e.g.~\cite{Shalev-Shwartz12survey,hazan2019introduction,cesa2006prediction}.
    If $p_i \in \{0,1\}$ for all $i$, then the learner is \emph{deterministic}, in which case $|y_i - p_i|$ is the binary indicator for whether the learner made a mistake.

We focus here on the $k$-realizable setting, which was suggested by~\cite{cesa1996line,cesa1997use} 
    and further studied by~\cite{abernethy2006continuous, mukherjee2010learning, branzei2019online}. 
    Here, the adversary must choose the answers so that at least one of the experts makes at most $k$ mistakes. 
    That is, there must exist an expert $j$ such that $y_i \neq \hat{y}_i^{(j)}$ for at most $k$ many indices $i$. 

The goal is to determine the optimal loss of the learner as a function of $n$ and $k$. 
    Let $\M^\star_D(n,k)$ denote the optimal loss of a deterministic learner and $\M^\star(n,k)$ 
    denote the optimal loss of a (possibly) randomized learner.\footnote{Note that we assume here that $k$ is known to the learner and that the horizon (i.e.\ number of rounds in the game) might be unbounded.
    In Section~\ref{sec:expert-intro-variations} below we explain how to extend our results to the complementing cases.}

\smallskip

The starting point is the basic fact\footnote{One might be tempted to interpret these inequalities as implying that $\M^\star(n,k)$ and $\M^\star_D(n,k)$ are nearly the same. However, the multiplicative gap of $1/2$ can be significant. 
For example, a randomized learner with a non-trivial error rate of $25\%$ corresponds to a deterministic learner with $50\%$ error-rate. The latter is trivially achieved by a random guess. For the same reason, sublinear regret guarantees can \underline{only} be achieved by randomized learners, although they are ``just'' a factor of $1/2$ better than deterministic learners, see e.g.~\cite{cesa2006prediction,Shalev-Shwartz12survey,hazan2019introduction}.} that
\begin{equation}\label{eq:detvsrand}
\frac{\M^\star_D(n,k)}{2}\leq \M^\star(n,k) \leq \M^\star_D(n,k).
\end{equation}
In their seminal work, Cesa-Bianchi et al.~\cite{cesa1997use} exhibited a randomized algorithm 
    which witnesses that in the regime when $k \gg \log n$ or $k \ll \log n$, 
    the lower bound in Equation~\ref{eq:detvsrand} is tight up to a relative factor of $o(1)$, 
    (See their Theorem 4.4.3). 

In a follow up work,~\cite{cesa1996line} aimed to find optimal deterministic and randomized algorithms. 
    They found a nearly optimal deterministic algorithm called \emph{binomial weights}, which is optimal (up to an additive constant) when $k$ is small enough. The main problem they left open is whether there is a randomized learner with loss exactly half the loss of their binomial weights algorithm (plus, maybe, a constant). 
 Below we show that the answer to this question is negative, and find tight guarantees on the second-order term in $\M^\star(n,k)$, in terms of the performances of their algorithm.

Nearly 10 years later, Abernathy, Langford and Warmuth~\cite{abernethy2006continuous} 
    showed that $\M^\star(n,k) \le \M^\star_D(n,k)/2 + C$ for every $k$ and every $n \geq N(k)$, 
    where $C$ is a universal constant (independent of $n,k$), thus showing that in the regime when $k=O(1)$
    the additive negligible term is indeed a universal constant (independent of $n,k$).

More recently, Br\^anzei and Peres~\cite{branzei2019online} showed 
    that $\M^\star(n,k) \leq (\frac{1}{2} + o(1))\M^\star_D(n,k)$ for~$k = o(\log n)$,
    while quantitatively improving upon the bounds given by~\cite{cesa1997use} in this regime.

In the next theorem we provide guarantees on $\M^\star(n,k)$
for \emph{all} $n \ge 2$ and $k\geq 0$, which are tight when $n=2$, thus fully resolving the question raised by~\cite{cesa1996line}.\footnote{When $n = 1$, $\M^\star(1,k) = \M^\star_D(1,k) = k$.}
Our lower bound
shows that the second-order term
tends to infinity when $n=2$.
The latter shows that the result by~\cite{abernethy2006continuous} does not apply for general $n,k$.

\begin{tcolorbox}
\begin{theorem}[Main Result (iv): Bounds for Randomized Predictors]\label{thm:expertIntro}
Let $\M^\star(n,k)$ denote the optimal expected mistake bound for prediction using expert advice in the $k$-realizable setting when there are $n$ experts, and let $D(n,k)$ denote the mistake bound of the binomial weights algorithm.
For all $n \geq 2$ and $k \geq 0$,
\[
 \M^\star(n,k) \leq \frac{D(n,k)}{2} + O\left(\sqrt{D(n,k)}\right).
\]
Furthermore, the error term cannot be improved
for $n = 2$:
\[
 \M^\star(2,k) = \frac{D(2,k)}{2} + \Omega\left(\sqrt{D(2,k)}\right).
\]
\end{theorem}
\end{tcolorbox}

We prove the upper bound in Section~\ref{sec:Mnk-upper}, and the lower bound in Section~\ref{sec:Mnk-lower}.
Both bounds are proved using the randomized $k$-Littlestone dimension. A special case of Theorem~\ref{thm:k-realizable-littlestone-bound-Intro} states that $\M^\star(n,k) = k + \Theta \mleft( \sqrt{k \log n} + \log n \mright)$, where the upper bound in this quantitative bound was first proved in \cite{cesa1997use}.

Using Theorem~\ref{thm:expertIntro} and the bounds of \cite{cesa1996line}, we obtain the following corollary.
\begin{corollary}\label{cor:expertIntro}
For all $n \geq 2$ and $k \geq 0$,
\[
\M^\star(n,k) = \mleft( \frac{1}{2} + o(1) \mright) \M^\star_D(n,k).
\]
\end{corollary}
This also follows from the results of \cite{vovk1990aggregating} for randomized learners.\footnote{In the COLT 2023 proceedings version of this paper, this corollary was unintentionally presented as brand new.}

\paragraph{Additional Related Work.}
Different variants of the experts problem have been extensively studied in the past 30 years and various techniques for bounding the optimal regret and mistake bounds were developed throughout the years, such as \emph{sequential Rademacher complexity} \cite{rakhlin2012relax, rakhlin2014online}, \emph{drifting games} \cite{mukherjee2010learning, luo2014drifting}, and the \emph{Hedge setting} \cite{abernethy2008optimal2, freund1997decision}. However, those techniques are seemingly tailored for randomized \emph{proper} learners (i.e., learners that predict using a distribution over the experts which is updated at the \underline{end} of each round), and proper learners are inherently suboptimal for the experts problem, even in the realizable case, as proven in \iffull\else the full version of this paper \cite{FullVersion}. \fi \cite{abernethy2008optimal2} identified the optimal algorithm for the hedge setting \cite{freund1997decision}, using a random walk analysis, which is similar to our characterization results. It will be interesting to investigate whether variations of these techniques can reproduce or even improve the bounds in this work.

\subsection{Variations}\label{sec:variations}

\subsubsection{Bounded Horizon} \label{sec:variations-bounded}

Consider learning $\cH$ in the $k$-realizable setting, and let $\M^\star_k=\M^\star_k(\cH)$ denote the optimal expected mistake bound.
    In particular, this means that the adversary can force~$\M^\star_k$ mistakes in expectation on any randomized learner. 
    This would be tolerable if in order to do so the adversary must use many examples, say $1000\M^\star_k$. 
    Indeed, this would mean that the learner makes only one mistake per a thousand examples (amortized), which is rather good.

This raises the question to what extent does $\M^\star_k$ capture
    the optimal mistake bound under the additional assumption that the horizon is bounded by a given $\rounds\in\mathbb{N}$.
    A bounded horizon is often assumed in the online learning literature, and in fact this question was explicitly asked by~\cite{cesa1997use} in the special case of prediction using expert advice.

Let $\M^\star_k(\rounds)$ denote the optimal expected mistake bound in the $k$-realizable setting with horizon bounded by $\rounds$.
    The following result shows that $\M^\star_k$ provides an excellent approximation of $\M^\star_k(\rounds)$;
    in particular, the scenario described above is impossible.

\smallskip
\begin{tcolorbox}
\begin{theorem}[Main Result (v): Bounded vs Unbounded Horizon]\label{thm:boundedhorizonIntro}
Let $\cH$ be an hypothesis class. Let $\M^\star_k$ denote the optimal expected mistake bound
in online learning $\cH$ in the $k$-realizable setting, and let $\M^\star_k(\rounds)$ 
denote the optimal expected mistake bound under the additional assumption that the input sequence has length at most $\rounds$.
Then,
\begin{enumerate}
    \item \label{itm:long-horizon} \textbf{Long horizon.} If $\rounds > 2\M^\star_k$ then
    \[
    \M^\star_k - \sqrt{2\M^\star_k \ln \M^\star_k} - 1 \leq \M^\star_k(\rounds) \leq \M^\star_k.
    \]
    \item \textbf{Short Horizon.} If $\rounds \leq 2\M^\star_k$ then
    \[\frac{\rounds}{2} - \sqrt{2\rounds\ln \rounds}-1 \leq \M^\star_k(\rounds) \leq \frac{\rounds}{2},\]
    and if $\rounds \leq \M^\star_k$ then $\M^\star_k(\rounds) =\frac{\rounds}{2}$.
\end{enumerate}
\end{theorem}
\end{tcolorbox}

    The upper bounds in Theorem~\ref{thm:boundedhorizonIntro} follow from basic facts: 
    indeed, $\M^\star_k(\rounds) \leq \M^\star_k$ holds because assuming a bounded horizon restricts the adversary, 
    and $\M^\star_k(\rounds) \leq \frac{\rounds}{2}$ follows by guessing each label uniformly at random. 
    The lower bounds are more challenging, and our proofs of them relies heavily on the randomized Littlestone dimension.

Our proof of Theorem~\ref{thm:boundedhorizonIntro} appears in Section~\ref{sec:shallow-trees-few-rounds}.
The proof relies on a simple extension 
    of our characterization to this setting:
    consider the following modification of the Littlestone dimension and its randomized variant:
\[
\LD_k(\cH,\rounds) = \sup_{\substack{T \text{ shattered} \\ \depth(T) \leq \rounds}} m_T
\quad \text{and} \quad
\RL_k(\cH,\rounds) =\frac{1}{2} \sup_{\substack{T \text{ shattered} \\ \depth(T) \leq \rounds}} E_T.
\]
The bounded randomized Littlestone dimension gives the precise mistake bound in this setting:
\begin{theorem}[Optimal Mistake Bounds: Bounded Horizon]\label{thm:characterization-roundsInt}
Let $\cH$ be an hypothesis class.
\begin{enumerate}
    \item The optimal deterministic mistake bound in online learning $\cH$ in the $k$-realizable setting with horizon $\rounds$ equals its bounded $k$-Littlestone dimension, $\LD_k(\cH,\rounds)$.\footnote{Trivially, $\LD_k(\cH,\rounds) = \min\{\rounds,\LD_k(\cH)\}$.}
    \item The optimal randomized mistake bound in online learning $\cH$ in the $k$-realizable setting with horizon $\rounds$ equals its bounded $k$-randomized Littlestone dimension, $\RL_k(\cH,\rounds)$.
\end{enumerate}
\end{theorem}
We prove Theorem~\ref{thm:characterization-roundsInt} in Section~\ref{sec:finite-horizon}. 

\paragraph{Prediction using Expert Advice.}

Also the problem of prediction using expert advice is often considered when the number of rounds is bounded (e.g.~\cite{cesa1997use}). 
Let $\M^\star(n,k,\expertrounds)$ be the optimal loss of the learner when the number of rounds is $\expertrounds$. 
By a simple reduction to Theorem~\ref{thm:boundedhorizonIntro} we show that
\[
\M^\star(n,k,\expertrounds)\approx 
\begin{cases}
   \M^\star(n,k) & \text{if } \expertrounds \geq 2\M^\star(n,k),\\
   \frac{\expertrounds}{2} & \text{if } \expertrounds < 2\M^\star(n,k).
\end{cases}
\]
The exact bounds are as in Theorem~\ref{thm:boundedhorizonIntro} when replacing $\M^\star(n,k,\expertrounds)$ and $\M^\star(n,k)$
with $\M^\star_k(\rounds)$ and~$\M^\star_k$.


\subsubsection{Adaptive Algorithms} \label{sec:expert-intro-variations}

The analysis in much of this work considers the case where the learning algorithm may depend explicitly on a bound $k$ on the number of mistakes of the best 
hypothesis (or expert).
However, it is also desirable to study mistake bounds achievable \emph{adaptively}: that is, by a single algorithm that applies to all $k$.
We present here one simple approach to obtaining such an algorithm, with a corresponding mistake bound.
However, the bound we obtain may likely be improvable, and generally we leave the question of obtaining a tightest possible adaptively-achievable mistake bound as an open problem.

\begin{theorem}
\label{thm:adaptive-algorithm-Intro}
There is an adaptive algorithm (i.e., which has no knowledge of $k^\star$)
such that, for every $k^\star$-realizable sequence for $\cH$, its expected number of mistakes is at most
\[
\M^{\star}_{k^\star} + O\!\left( \sqrt{ \M^{\star}_{k^\star} \log\bigl((k^\star+1) \log \M^{\star}_{k^\star}\bigr)} \right).
\]
\end{theorem}

In the special case of the general \emph{experts} 
setting, 
since we know that $\M^{\star}(n,k^\star) = \Omega( k^\star + \log(n) )$, 
we obtain the following bound on the expected number of mistakes:
\[
\M^{\star}(n,k^\star) + O\!\left( \sqrt{ \M^{\star}(n,k^\star) \log \M^{\star}(n,k^\star) } \right) = (1 + o(1)) \M^{\star}(n,k^\star).
\]
In particular, combining this with Theorem~\ref{thm:expertIntro}, 
we find that this algorithm adaptively still 
achieves an expected number 
of mistakes $\left(\frac{1}{2}+o(1)\right) \M^{\star}_{D}(n,k^\star)$.

On the other hand, in the case of concept classes $\cH$ 
with a bounded Littlestone dimension~$\LD(\cH)$, 
we know from Theorem~\ref{thm:k-realizable-littlestone-bound-Intro} that 
\[
\M^{\star}_{k^\star} \leq k^\star + O\!\left( \sqrt{ k^\star \LD(\cH) } + \LD(\cH) \right).
\]
Theorem~\ref{thm:adaptive-algorithm-Intro} implies that
the adaptive procedure nearly preserves the form of this upper bound, guaranteeing a slightly larger bound of the form
\[
k^\star + O\!\left( \sqrt{ k^\star \LD(\cH) \log\!\left( k^\star \log \LD(\cH) \right) } + \LD(\cH) \right).
\]

Our proof of Theorem~\ref{thm:adaptive-algorithm-Intro} appears in 
Section~\ref{sec:adaptive-algorithm}.
The adaptive technique we propose involves using an experts 
algorithm of \cite{koolen:15} named $\Squint$, with experts 
defined by the optimal randomized algorithm for the $k$-realizable 
setting, for all values of $k$.

\section{Technical Overview}
\label{sec:intro-technical-overview}

In its greatest generality, online prediction is a game involving two randomized parties, an adversary who is producing examples, and a learner who is trying to correctly predict the labels of all or most of these examples. In the realizable case, the adversary is moreover constrained by an hypothesis class which must be adhered to.

Various techniques are used in the literature to analyze this sophisticated setting. On the one hand, learning rules show which hypothesis classes lend themselves to learning, and on the other hand, strategies for the adversary put limitations on what can be learned, and at what cost.

In this work, we identify the combinatorial core behind many settings of online learning. In this, we follow up on Nick Littlestone's classical work on deterministic online learning, as well as on other classical work in learning theory such as that the foundational work of Vapnik and Chervonenkis.

Reducing the messy probabilistic setting of online learning to the clean combinatorial setting of shattered trees enables us to tackle open questions about prediction using expert advice, which are hard to approach directly.

\subsection{Combinatorial Characterizations}
\label{sec:technical-oveview-dimension}

The Littlestone dimension of an hypothesis class $\cH$ is the maximal depth of a complete binary tree which is shattered by $\cH$. A tree of depth $D$ easily translates into a strategy for the adversary which forces the learner to make $D$ mistakes. In other words, a tree shattered by $\cH$ is an obvious obstacle to learning $\cH$.

The magic of Littlestone dimension is the opposite direction: Littlestone's $\SOA$ learning rule makes at most $\LD(\cH)$ mistakes, showing that trees shattered by $\cH$ are the \emph{only} obstacle for learning $\cH$. This is a common phenomenon in mathematics: an obvious necessary condition turns out to be (less obviously) sufficient.

\paragraph{Defining the randomized Littlestone dimension.}

In order to motivate the definition of the randomized Littlestone dimension, let us first examine the (deterministic) Littlestone dimension. Given a tree $T$ shattered by $\cH$, the adversary executes the following strategy, starting at the root:

\begin{center}
    \emph{At an internal node labeled $x$, ask the learner for the label of $x$, and follow the opposite edge.}
\end{center}

This strategy follows a branch of $T$, and forces the learner to make a mistake in each round. The total number of mistakes which the adversary can guarantee is precisely $m_T$, the minimum length of a branch in $T$. The resulting input sequence is realizable by $\cH$ since $T$ is shattered by~$\cH$.

The definition of the randomized Littlestone dimension follows a similar approach, but uses a different strategy for the adversary:

\begin{center}
    \emph{At an internal node labeled $x$, ask the learner for the label of $x$, and follow a \emph{random} edge.}
\end{center}

This strategy also follows a branch of $T$, and it forces the learner to make \emph{half} a mistake in each round, in expectation.\footnote{Recall that we model a randomized learner as a learner which makes a ``soft'' prediction $p \in [0,1]$; if the true label is $y$, then the learner's loss is $|p-y|$. When we choose the label $y$ at random, the expected loss is $\mathbb{E}[|p-y|]=\tfrac12$ regardless of $p$.} The total expected number of mistakes is $E_T/2$, where $E_T$ is the expected length of a random branch of $T$.

We define the randomized Littlestone dimension by considering all such adversary strategies:
\[
  \RL(\cH) = \frac{1}{2} \sup_{T \text{ shattered}} E_T.
\]

\paragraph{Extending the Standard Optimal Algorithm.}

Littlestone's Standard Optimal Algorithm ($\SOA$) makes at most $\LD(\cH)$ mistakes on any realizable input sequence. The algorithm is very simple. It maintains a subset $V$ of $\cH$ which consists of all hypotheses which are consistent with the data seen so far. Given a sample $x$, one of the following must hold, where $V_{x \to y}$ is the subset of $V$ consisting of all hypotheses assigning to $x$ the label $y$:
\begin{enumerate}
    \item $\LD(V_{x \to 0}) < \LD(V)$. The learner predicts $\hat{y} = 1$.
    \item $\LD(V_{x \to 1}) < \LD(V)$. The learner predicts $\hat{y} = 0$.
\end{enumerate}
One of these cases must hold, since otherwise we could construct a tree of depth $\LD(V) + 1$ shattered by $V$. Each time that the learner makes a mistake, $\LD(V)$ decreases by $1$, and so the learner makes at most $\LD(\cH)$ mistakes.

Our randomized extension of $\SOA$, which we call $\RSOA$, follows a very similar strategy. It maintains $V$ in the same way. Given a sample $x$, we want to make a prediction $p$ which ``covers all bases'', that is, results in a good outcome for the learner whatever the correct label $y$ is. Given a prediction $p$, the adversary can guarantee a loss of
\[
 \max \{ p + \RL(V_{x \to 0}), 1 - p + \RL(V_{x \to 1}) \}.
\]
For the optimal choice of $p$, this quantity is at most $\RL(V)$, as we show in Section~\ref{sec:thm:characterization}.

\paragraph{The $k$-realizable setting and weighted $\SOA$.}
The $k$-realizable setting is handled similarly. In the definition of randomized Littlestone dimension, instead of requiring the tree to be shattered, it suffices for it to be $k$-shattered, since the adversary need only produce an input sequence which is $k$-realizable.

The main novelty in this setting is a \emph{weighted} analog of the $\SOA$ learning rule. 
This weighted $\SOA$ rule relates to the classical $\SOA$ in a similar way like the \emph{Weighted Majority} algorithm relates to \emph{Halving}.
In particular, it keeps track, for each hypothesis, how many more mistakes are allowed. Accordingly, we consider the more generalized setting of \emph{weighted hypothesis classes}. These are hypothesis classes in which each hypothesis has a ``mistake budget''. The definition of randomized Littlestone dimension extends to this setting, and allows us to generalize $\RSOA$ to the randomized agnostic setting.



\subsection{Quasi-balanced Trees}
\label{sec:technical-overview-quasibalanced}

Given an hypothesis class $\cH$, how does an optimal strategy for the adversary look like? Such a strategy must make the analysis of $\RSOA$ tight, and in particular, if the first sample it asks is $x$, then
\[
 \RL(\cH) = p + \RL(\cH_{x \to 0}) = 1 - p + \RL(\cH_{x \to 1}),
\]
where $p$ is the prediction of the learner.\footnote{Strictly optimal strategies do not always exist, and even when they do, they might require an unbounded number of rounds. For the sake of exposition we gloss over these difficulties.}

The strategy of the adversary naturally corresponds to a tree which is shattered by $\cH$: the root is labeled $x$, and the edge labeled $y$ leads to a tree corresponding to an optimal strategy for $\cH_{x \to y}$. Suppose that we further assign weights to the edges touching the root: the $0$-edge gets the weight $p$, and the $1$-edge gets the weight $1-p$. If we assign weights to the remaining edges recursively then the resulting tree satisfies the following property:

\begin{center}
    \emph{Every branch has the same total weight $\RL(\cH)$.}
\end{center}

More generally, a tree $T$ is \emph{quasi-balanced} if we can assign non-negative weight to its edges such that (i) the weights of the two edges emanating from a vertex sum to $1$, and (ii) all branches have the same total weight (which must be $E_T/2$). If a tree is quasi-balanced then the weight assignment turns out to be \emph{unique}.

A tree in which all branches have the same depth is quasi-balanced, but the class of quasi-balanced trees is a lot richer, including for example the path appearing in Figure~\ref{fig:finite-path-intro}.

\begin{figure}
    \centering
    \begin{forest}
[, circle, draw
    [, circle, draw, edge label={node[midway,left] {$\tfrac18$}}
        [, circle, draw, edge label={node[midway,left] {$\tfrac14$}}
            [, circle, draw, edge label={node[midway,left]{$\tfrac12$}}]
        [,circle, draw, edge label={node[midway,right] {$\tfrac12$}}]]
    [,circle, draw, edge label={node[midway,right] {$\tfrac34$}}]]
[,circle, draw, edge label={node[midway,right] {$\tfrac78$}}]]
\end{forest}
        \caption{A quasi-balanced tree. The edges are labeled with the unique weights. The sum of weights in each branch is $\tfrac78$, which is half the expected branch length $\tfrac74$.}
    \label{fig:finite-path-intro}
\end{figure}

There is a simple criterion for quasi-balancedness:

\smallskip

\begin{tcolorbox}
A tree $T$ is quasi-balanced if and only if it is \emph{monotone}: if $w$ is a descendant of $v$ then $E_{T_w} \leq E_{T_v}$, where $T_u$ is the subtree rooted at $u$.
\end{tcolorbox}

Since the loss guaranteed by an adversary following the strategy corresponding to a tree $T$ is $E_T/2$, it is clear that the best strategy is always monotone. This argument shows that
\[
 \RL(\cH) = \frac12 \sup_{T \text{ shattered, monotone}} E_T.
\]
In other words, it suffices to consider only quasi-balanced trees when defining the randomized Littlestone dimension. This is the randomized counterpart of a trivial property of the Littlestone
dimension: in order to define the Littlestone dimension, it suffices to consider \emph{balanced} trees, that is, trees in which all branches have the same length. We can view quasi-balancedness as a relaxation of strict balancedness.

\paragraph{Concentration of expected branch length.}

The randomized Littlestone dimension is defined in terms of the expected branch length. However, several of our results require knowledge of the distribution of the branch length.

For example, Theorem~\ref{thm:boundedhorizonIntro} states that $2\RL(\cH) + O(\sqrt{\RL(\cH)}\log(\RL(\cH)/\epsilon))$ rounds are needed in order for the adversary to guarantee a loss of $\RL(\cH) - \epsilon$. The number of rounds corresponds to the depth of the tree, and so the natural way to prove such a result would be to start with a tree~$T$ satisfying $E_T/2 = \RL(\cH)$, and prune it to depth $2\RL(\cH) + O(\sqrt{\RL(\cH)}\log(\RL(\cH)/\epsilon))$. We would like to say that this does not reduce the expected branch length by much, since the length of most branches does not exceed $E_T$ by much. Other applications such as prediction using expert advice need concentration from the other side (the length of most branches does not fall behind $E_T$ by much).

It is possible to construct trees for which the length of a random branch isn't concentrated around its expectation. For example, we can take an infinite path which, every so often, splits into a deep complete binary tree. If we are careful, we can guarantee that the expected branch length is finite but its variance is infinite.

At this point, quasi-balancedness comes to the rescue. The monotonicity property of quasi-balanced trees implies that the choice of an edge at every step of a random branch does not affect the final length by much. Consequently, Azuma's inequality (a version of Chernoff's inequality for martingales) shows that for quasi-balanced trees, the length of a random branch is strongly concentrated around its expectation. This simple observation drives several of our strongest results.

\subsection{Prediction using Expert Advice (techniques)}
\label{sec:technical-overview-experts}

At first, the setting of prediction using expert advice looks similar, but not identical, to our setting. However, it turns out that it is actually a \emph{special case} of our setting, for a specific hypothesis class known as the \emph{universal hypothesis class} $\cU_n$.

The class $\cU_n$ contains $n$ different hypotheses, which correspond to the experts. For each possible set of predictions $\hat{y}^{(1)},\ldots,\hat{y}^{(n)}$ there is a corresponding element in the domain. In other words, the domain is $\cX = \{0,1\}^n$, and the hypotheses in $\cU_n$ are the $n$ projections $h_i(x_1,\ldots,x_n) = x_i$.

With this equivalence in place, we can apply the theory we have developed so far to analyze prediction using expert advice.
Our main result concerning this setting, Theorem~\ref{thm:expertIntro}, consists of an upper bound on $\M^\star(n,k)$, and a lower bound on $\M^\star(2,k)$.

We start with the upper bound on $\M^\star(n,k)$.
In view of the equivalence above, we want to bound the expected branch length of any tree $T$ which is $k$-shattered by $\cU_n$. We can assume that $T$ is quasi-balanced, and so the length of a random branch of $T$ is roughly $E_T$.
If $T$ were strictly balanced, then a random branch would be $k$-realizable by $\cU_n$ with probability at most
\[
 n \frac{\binom{E_T}{\leq k}}{2^{E_T}}.
\]
\cite{cesa1996line} have shown that the largest value of $E_T$ for which this quantity is at least $1$, which we denote by $D(n,k)$,
provides the state-of-the-art upper bound on $\M^\star_D(n,k)$.
Since $T$ is only quasi-balanced, we get a slightly worse bound. 

A nice proof of the lower bound on $\M^\star(2,k)$ is given by identifying the optimal tree. Intuitively, it seems obvious that rounds in which both experts make the same prediction are ``wasteful'', and we can show this formally. By symmetry, we can assume that the first expert always predicts $0$ and that the second expert always predicts $1$. We can construct the corresponding tree explicitly, and conclude that
\[
 \M^\star(2,k) = k + \frac{(k+1/2) \binom{2k}{k}}{4^k}.
\]
\iffull \else The proof of this result can be found in the full version \cite{FullVersion} of this paper.\fi
\section{Background and Basic Definitions}\label{sec:prelim}


Unless stated otherwise, our logarithms are base~$2$. A summary of the paper's notation may be found in Table~\ref{tab:notation}.

\paragraph{Online Learning.}
Let $\cX$ be a set called the \emph{domain}, and $\cY$ be a set called the \emph{label set}. 
    In this work we focus on \emph{binary classification}, and thus $\cY= \{0,1\}$. 
    A pair $(x,y)\in \cX \times \cY$ is called an \emph{example}, and 
    an element $x\in \cX$ is called an \emph{instance} or an \emph{unlabeled example}. 
    A function $h\colon \cX\to\cY$ is called a \emph{hypothesis} or a \emph{concept}.
    A \emph{hypothesis class}, or a \emph{concept class}, is a set $\cH \subset \cY^{\cX}$.
    A sequence of examples  $S=\{(x_i,y_i)\}_{i=1}^t$ is said to be \emph{realizable} by $\cH$ if there exists $h \in \cH$ such that $h(x_i) = y_i$ for all $1 \leq i \leq t$.



Online learning~\cite{shalev2014understanding, cesa1997use} is a repeated game between a learner and an adversary.
Each round $i$ in the game proceeds as follows:
\begin{enumerate}[(i)]
\item The adversary sends the learner an unlabeled example $x_i \in \cX$.
\item The learner predicts a value $p_i \in [0,1]$ and reveals it to the adversary.
\item The adversary reveals the true label $y_i$, and the learner suffers the \emph{loss} $|y_i - p_i|$.
\end{enumerate}
The value $p_i$ should be understood as the probability (over the learner's randomness) of predicting~$y_i=1$. Notice that the adversary only gets to see $p_i$, which reflects the assumption that the adversary does not know the learner's internal randomness. 
Notice also that the suffered loss $\lvert y_i - p_i\rvert$ exactly captures the probability that the learner makes a mistake.
The above is a standard way to model randomized learners in online learning, see e.g.~\cite{Shalev-Shwartz12survey}.
If $p_i \in \{0,1\}$ for all $i$, then the learner is \emph{deterministic}, in which case $|y_i - p_i|$ is the binary indicator for whether the learner made a mistake. 


We model learners as functions $\Lrn \colon (\cX \times \cY)^\star \times \cX \rightarrow [0,1]$.
Given a learning rule $\Lrn$ and an input sequence of examples $S = (x_1,y_1),\ldots,(x_t,y_t)$, 
    we denote the (expected) number of mistakes $\Lrn$ makes on $S$ by
    \[\M(\Lrn; S) = \sum_{i=1}^t \lvert y_i - p_i \rvert,\]
    where $p_i = \Lrn((x_1,y_1),\ldots,(x_{i-1},y_{i-1}),x_i)$ is the prediction of the learner on the $i$'th example.
    
An hypothesis class $\cH$ is \emph{online learnable} (or \emph{learnable}) if there exists a finite bound $M$ and a learning rule $\Lrn$ such that for any input sequence $S$ 
which is realizable by $\cH$ it holds that $\M(\Lrn; S) \leq M$. We define the \emph{optimal} randomized mistake bound of $\cH$ to be
\begin{equation}\label{eq:optrand}
\M^{\star}(\cH)=\adjustlimits\inf_ {\Lrn} \sup_{S} \M(\Lrn ;S)
\end{equation}
where the infimum is taken over all learning rules, and the supremum is taken over all realizable input sequences $S$.

We denote by $\M^\star_D(\cH)$ the optimal deterministic mistake bound of $\cH$.
That is, $\M^\star_D(\cH)$ is defined in the same way as $\M^{\star}(\cH)$, with the additional restriction that $\Lrn$ must be deterministic (that is, the output must be in $\{0,1\}$).

When $\cH = \emptyset$, the set of realizable input sequences is empty, and therefore the supremum is not defined. It is technically convenient to deal with this special case by defining $\M^\star_D(\emptyset) = \M^\star(\emptyset) = -1$. When the context is clear, we may sometimes refer to the deterministic or randomized mistake bound as the \emph{accumulating loss} of the learner through the entire game, or simply as the learner's \emph{loss} through the entire game.

\paragraph{Decision Trees and the Littlestone Dimension.}
In this paper, a \emph{tree} $T$ refers to a finite full rooted ordered binary tree (that is, a rooted binary tree where each node which is not a leaf has a left child and a right child), equipped with the following information:
\begin{enumerate}
    \item Each internal node $v$ is associated with an instance $x \in \cX$.
    \item For every internal node $v$, the left outgoing edge is associated with the label $0$, and the right outgoing edge is associated with the label $1$.
\end{enumerate}


We stress that by default, the trees we consider are finite and their vertices are labeled. Whenever we consider infinite trees or unlabeled trees, we specifically mention these attributes.

\smallskip


The tree is directed from the root towards the leaves. 

A \emph{prefix} of the tree $T$ is any path that starts at the root. In this paper, a path is defined by a sequence of consecutive vertices. If a path is not empty, we may refer it by the sequence of consecutive edges corresponding with the sequence of consecutive vertices defining it. 
    A prefix $v_0,v_1, \dots, v_t$ defines a sequence of examples $(x_1,y_1), \dots, (x_{t},y_{t})$ in a natural way: 
    for every $i \in [t]$, $x_i$ is the instance corresponding to the node $v_{i-1}$, 
    and $y_i$ is the label corresponding to the edge $v_{i-1}\to v_{i}$. 
    A prefix is called \emph{maximal} if it is maximal with respect to containment, 
    that is, there is no prefix in the tree that strictly contains it. 
    This is equivalent to requiring that $v_{t}$ be a leaf. 
    A maximal prefix is called a \emph{branch}, and the set of branches of $T$ is denoted by $B(T)$. 
    The length of a prefix is the number of edges in it (so, the length is equal to the size of the corresponding sequence of examples).

A prefix in the tree is said to be \emph{realizable} by $\cH$ if the corresponding sequence of examples is realizable by $\cH$.
    A tree $T$ is \emph{shattered} by $\cH$ if all branches in $T$ are realizable by $\cH$.
    The \emph{Littlestone dimension} of an hypothesis class $\cH$, denoted by $\LD = \LD(\cH)$,  is the maximal depth of a \emph{complete} (also known as \emph{perfect}, or \emph{balanced}) binary tree (that is, a tree in which all branches have the same depth) shattered by $\cH$ if $\cH \neq \emptyset$, and $-1$ when $\cH = \emptyset$. 
    If the maximum does not exist, then $\LD=\infty$.

\paragraph{Littlestone Dimension $\equiv$ Optimal Deterministic Mistake Bound.}    
In his seminal work from 1988, Nick Littlestone proved that the optimal mistake bound of a deterministic learner is characterized by the Littlestone dimension:
\begin{theorem}[Optimal Deterministic Mistake Bound~\cite{littlestone1988learning}]\label{thm:Littlestone}
Let $\cH$ be an hypothesis class. Then, $\cH$ is online learnable if and only if $\LD(\cH)<\infty$.
Further, the optimal deterministic mistake bound satisfies $\M^{\star}_D(\cH)=\LD(\cH)$.
\end{theorem}

\paragraph{Doob's Exposure Martingales.}
Let $f\colon \{0,1\}^\mathbb{N} \to \mathbb{R}$. Consider the random variable $X=f(\Vec{b})$, where $\Vec{b}$ is sampled uniformly at random. Define a sequence $L_0, L_1, L_2, \dots$, each defined by $L_i = \mathbb{E}[X|b_1,  \dots, b_{i-1}]$ (so $L_0 = \mathbb{E}[X]$). The sequence $L_0, L_1, L_2, \ldots$ is called an \emph{exposure martingale}. It is well-known that an exposure martingale is indeed a martingale~\cite{doob1953stochastic}.
\begin{table}
\begin{center}
    \begin{tabular}{|l|l|}
        \hline
        \textbf{Notation} & \textbf{Meaning} \\
        \hline
        $\cX$ & An instance domain \\
        $\cH \subset \{0,1\}^\cX$ & A hypothesis class \\
        $\M^{\star}(\cH)$ & The optimal randomized mistake bound of $\cH$ \\
        $\M^{\star}_D(\cH)$ & The optimal deterministic mistake bound of $\cH$ \\
        $\RL(\cH)$ & The randomized Littlestone dimension of $\cH$ \\
        $\LD(\cH)$ & The Littlestone dimension of $\cH$ \\
        $\M^{\star}(\cH, \rounds)$ & The optimal randomized mistake bound of $\cH$ with horizon $\rounds$ \\
        $\RL(\cH, \rounds)$ & The bounded randomized Littlestone dimension of $\cH$ with trees of depth $\leq \rounds$ \\
        $\M^{\star}(\cH,k)$ & The $k$-realizable optimal randomized mistake bound of $\cH$\\
        $\M^{\star}_D(\cH,k)$ & The $k$-realizable optimal deterministic mistake bound of $\cH$\\
        $\RL_k(\cH)$ & The $k$-randomized Littlestone dimension of $\cH$ \\
        $\LD_k(\cH)$ & The $k$-Littlestone dimension of $\cH$ \\
        $\M^\star(n,k)$ & The $k$-realizable optimal randomized mistake bound of the class of $n$ experts \\
        $\M^\star_D(n,k)$ & The $k$-realizable optimal deterministic mistake bound of the class of $n$ experts \\
        \hline
    \end{tabular}
\end{center}
\caption{Summary of notation. Some of the notation is defined in the following sections of the paper.}
\label{tab:notation}
\end{table}

\section{Randomized Littlestone Dimension and Optimal Expected Mistake Bound}

In this section we study the randomized Littlestone dimension.
\iffull
\else
We shall define the randomized Littlestone dimension and prove that it characterizes the optimal randomized mistake bound exactly.
\fi

The randomized Littlestone dimension is defined using trees, which correspond to strategies of the adversary.
We study a special class of trees, \emph{quasi-balanced trees}, in Section~\ref{sec:quasi-balanced-trees}.
\iffull
\else
Such trees define optimal strategies for the adversary (more details on such strategies are found in the full version of this paper \cite{FullVersion}).
\fi
\iffull
\else
Several other applications of quasi-balanced trees are presented in the full version of this paper \cite{FullVersion}; more applications are found throughout the paper.
\fi

The first main contribution of this paper is a characterization of the optimal randomized mistake bound in terms of a combinatorial parameter
we call the \emph{randomized Littlestone dimension} and denote by $\RL = \RL(\cH)$.

We define $\RL(\cH)$ using a natural distribution on the branches of trees (a branch is a root-to-leaf path). Given a tree $T$, a \emph{random branch} is chosen by starting at the root, and at each step, picking an edge leaving the current vertex uniformly at random, until reaching a leaf. We denote the expected length of a random branch by $E_T$. It is given explicitly by the formula
\[
 E_T = \sum_{b \in B(T)} |b| \cdot 2^{-|b|},
\]
where $B(T)$ is the set of branches of $T$.
If we think of a random branch as a distribution over $B(T)$, then $E_T$ is its entropy.

It is convenient to define the length of the empty branch to be $-1$.
With this convention, the expected branch length in $T$ satisfies the  recursion
\begin{equation} \label{eq:expected_branch_recursive}
E_T = 1 + \frac{E_{T_0} + E_{T_1}}{2},
\end{equation}
where $T_0,T_1$ are the subtrees of the root of $T$, which are empty when $T$ is a leaf.



\begin{definition}[Randomized Littlestone Dimension] \label{def:randomized-littlestone-dimension}
Let $\cH$ be an hypothesis class. 
The \emph{randomized Littlestone dimension} of $\cH$, denoted by $\RL(\cH)$, is defined by
\[
\RL(\cH) =\frac{1}{2} \sup_{T \text{ shattered}} E_T.
\]
In the special case when $\cH=\emptyset$, define $\RL(\cH)=-1$.
\end{definition}

To compare $\RL(\cH)$ with $\LD(\cH)$, let us consider the following equivalent way of defining $\LD(\cH)$:
\[
\LD(\cH) = \sup_{T \text{ shattered}} m_T,
\]
where $m_T$ is the minimum length of a branch in $T$.
    Thus, the difference is that in $\RL(\cH)$ we take the expected depth rather than the minimal depth,
    and multiply by a factor of $1/2$.
    
\begin{theorem}[Optimal Randomized Mistake Bound]\label{thm:characterization}
Let $\cH$ be an hypothesis class. Then,
\[
\M^{\star}(\cH) = \RL(\cH).
\]
\end{theorem}

We prove the theorem in Subsection~\ref{sec:thm:characterization} using \emph{randomized $\SOA$}, a randomized adaptation of Littlestone's classical $\SOA$ algorithm. This shows that the infimum in Equation~\eqref{eq:optrand} is realized by a minimizer.

\iffull
\else
\subsection{Proof of Characterization} \label{sec:thm:characterization}
\fi 
The case $\cH = \emptyset$ holds by definition. Therefore we assume that $\cH \neq \emptyset$.
The lower bound ``$\RL(\cH)\leq \M^\star(\cH)$'' boils down to the following lemma:
\begin{lemma}\label{lem:lower_bound_help}
Let $\cH$ be an hypothesis class, and let $T$ be a finite tree which is shattered by $\cH$. 
Then, for every learning rule $\Lrn$ there exists a realizable sequence $S$ so that $\M(\Lrn ; S)\geq E_T/2$.
Moreover, there exists such a sequence $S$ which corresponds to one of the branches of $T$.
\end{lemma}



\begin{proof}
The proof is given by a simple probabilistic argument. Suppose that we pick a random branch in the tree according to the random branch distribution: begin at the root, pick a random child of the root uniformly at random, and recursively pick a random branch in the corresponding subtree.
Consider the random variable
\[L_{T} = \M(\Lrn ; S),\]
where $S$ is the sequence of examples corresponding to a random branch drawn as above. 
It suffices to show that $\mathbb{E}[L_{T}] = E_T/2$. We prove this by induction on the depth of $T$.

In the base case, $T$ is a single leaf, and there are no internal nodes. Hence $S$ is always the empty sequence, and $\mathbb{E}[L_T] = 0 = E_T/2$, as required.

For the induction step, let $T_0$ and $T_1$ be the left and right subtrees of $T$, respectively. 
The expected loss of $\Lrn$  on the first example in $S$ is~$1/2$, because the label $y\in\{0,1\}$ is chosen uniformly at random, independently of the learner's prediction (formally, $\frac{|0-p| + |1-p|}{2} = 1/2$ for all $p \in [0,1]$). Therefore, by linearity of expectation,
\begin{align*}
\mathbb{E}[L_T] &= \frac{1 + \mathbb{E}[L_{T_0}] + \mathbb{E}[L_{T_0}]}{2} \\
                &= \frac{1 + E_{T_0}/2 + E_{T_1}/2}{2} \tag{\text{by the induction hypothesis}}\\
                &= E_T/2, \tag{by  Eq. \eqref{eq:expected_branch_recursive}}
\end{align*}
as required.
\end{proof}
By applying Lemma~\ref{lem:lower_bound_help} on every shattered tree and taking the supremum, we conclude the lower bound:
\begin{corollary}[Lower bound] \label{cor:characterization_upper}
For every hypothesis class $\cH$ it holds that $\M^\star(\cH) \geq \RL(\cH)$.
\end{corollary}


We now turn to prove the upper bound ``$\RL(\cH) \geq \M^\star(\cH)$''. This is achieved via the $\RSOA$ learning rule, described in Figure~\ref{fig:RSOA}. 

We begin with the following useful property of $\RL$: 

\begin{observation} \label{obs:expected_prop}
Let $\cH$ be a non-empty hypothesis class. Then, 
\[\RL(\cH) = \frac{1}{2}\sup_{x\in \cX}\bigl(1 + \RL(\cH_{x\to 0}) + \RL(\cH_{x\to 1})\bigr).
\]
\end{observation}
\begin{proof}
Observation~\ref{obs:expected_prop} follows from Equation~\eqref{eq:expected_branch_recursive}:
let $\cS(\cH)$ denote the set of trees that are shattered by $\cH$, and for $x\in \cX$, let $\cS_x(\cH)\subseteq \cS(\cH)$ denote the set of trees that are shattered by $\cH$ whose root is labeled by $x$.
Then,
\[
\RL(\cH) = \frac{1}{2}\sup_{T\in \cS(\cH)} E_T  
           = \frac{1}{2}\adjustlimits\sup_x \sup_{T\in \cS_x(\cH)} E_T.
\]
By Equation~\eqref{eq:expected_branch_recursive},
\[\sup_{T\in \cS_x(\cH)} E_T =  1 + \frac{\sup_{T_1\in\cS(\cH_{x\to 1})}E_{T_1} + \sup_{T_0\in\cS(\cH_{x\to 0})}E_{T_0}}{2} =  
1 + \RL(\cH_{x\to 1}) + \RL(\cH_{x\to 0}),\]
which finishes the proof.
\end{proof}
Notice that the classical Littlestone dimension satisfies a similar recursion:
\[\LD(\cH) = \sup_{x\in \cX}\bigl(1+ \min\{\LD(\cH_{x\to 1}),\LD(\cH_{x\to 0})\}\bigr).
\]
The following lemma is the crux of the analysis:
it guides the choice of  the prediction $p_i$ in each round.
\begin{figure}
    \centering
    \begin{tcolorbox}
    \begin{center}
        $\RSOA$: \textsc{Randomized $\SOA$}
    \end{center}
    \textbf{Input:} An hypothesis class $\cH$.
    \\
    \textbf{Initialize:} Let $V^{(1)} = \cH$.\\
    \\
    \textbf{For $i=1,2,\dots$}
    \begin{enumerate}
        \item Receive $x_i$. 
        \item \label{itm:rand_pred} Predict $p_i \in [0,1]$ such that the value 
        \begin{equation} \label{eq:rand_pred}
        \max \mleft \{p_i + \RL\mleft(V^{(i)}_{x_i \to 0}\mright), 1 - p_i + \RL\mleft(V^{(i)}_{x_i \to 1}\mright) \mright \}   
        \end{equation}
        is minimized, where $V^{(i)}_{x_i\to b} = \{h\in V^{(i)} : h(x_i)=b\}$.
        \item Receive true label $y_i$.
        \item Update $V^{(i+1)} = V^{(i)}_{x_i \to y_i}$.
    \end{enumerate}
    \end{tcolorbox}
    \caption{The randomized $\SOA$ is a variation of $\SOA$ that finds an optimal randomized prediction in every round. $\SOA$ is the name of the original deterministic algorithm by Littlestone \cite{littlestone1988learning}, and it stands for ``Standard Optimal Algorithm".}
    \label{fig:RSOA}
\end{figure}

\begin{lemma}[Optimal prediction for each round] \label{lem:rand_pred}
Let $\cH$ be an hypothesis class, and let $x\in \cX$. Then there exists $p \in [0,1]$ so that
\[p + \RL\mleft(\cH_{x\to 0}\mright) \leq \RL(\cH) \quad \text{and}\quad (1 - p) + \RL\mleft(\cH_{x\to 1}\mright) \leq \RL(\cH).\]
\end{lemma}


\begin{proof}[Proof of Lemma~\ref{lem:rand_pred}]
If $\RL(\cH) = \infty$ then the lemma is trivial. Therefore we assume that $\RL(\cH) < \infty$. Assume first that $ \mleft \lvert \RL\mleft(\cH_{x\to 0}\mright) - \RL\mleft(\cH_{x\to 1}\mright) \mright \rvert > 1$.
If $\RL\mleft(\cH_{x\to 0}\mright) + 1 < \RL\mleft(\cH_{x\to 1}\mright)$, then by choosing $p=1$ and applying the fact that $\RL(\cH') \leq \RL(\cH)$ if $\cH' \subseteq \cH$ we get
\begin{align*}
    p + \RL\mleft(\cH_{x\to 0}\mright) &= 1 + \RL\mleft(\cH_{x\to 0}\mright) < \RL\mleft(\cH_{x\to 1}\mright) \leq  \RL(\cH),\\
    1 - p + \RL\mleft(\cH_{x\to 1}\mright) &= \RL\mleft(\cH_{x\to 1}\mright) \leq \RL(\cH),
\end{align*}
as desired. The case $\RL\mleft(\cH_{x\to 1}\mright) + 1 < \RL\mleft(\cH_{x\to 0}\mright)$ is treated similarly. 

It remains to handle the case when $\mleft \lvert \RL\mleft(\cH_{x\to 0}\mright) - \RL\mleft(\cH_{x\to 1}\mright) \mright \rvert \leq 1$. 
Set 
\[p := \frac{1 + \RL\mleft(\cH_{x\to 1}\mright) - \RL\mleft(\cH_{x\to 0}\mright) }{2}.\] 
By assumption, $p\in [0,1]$, and also
\begin{align*}
    p + \RL\mleft(\cH_{x\to 0}\mright)
    &=
    1 - p + \RL\mleft(\cH_{x\to 1}\mright) \\
    &=
    \frac{1 + \RL\mleft(\cH_{x\to 0}\mright) + \RL\mleft(\cH_{x\to 1}\mright)}{2} \\
    & \leq \RL(\cH) \tag{Observation~\ref{obs:expected_prop}}.
\end{align*}

\end{proof}

\begin{lemma}[Upper bound] \label{lem:characterization_upper}
Let $\cH$ be an hypothesis class. 
Then the $\RSOA$ learner described in Figure~\ref{fig:RSOA} has expected mistake bound 
\[
\M(\RSOA ; S) \leq \RL(\cH)
\]
for every realizable input sequence $S$.
\end{lemma}

\begin{proof}
The proof is by induction on the length of the input sequence. Let $S = (x_1,y_1), \dots, (x_t,y_t)$ be a realizable sequence. In the base case $t=0$ we have $\M(\RSOA ; S) = 0 \leq \RL(\cH)$. For the induction step, assume that $t \geq 1$, and let $S' = (x_2,y_2), \dots, (x_t,y_t)$ be the input sequence without the first example. In the first round, the learner predicts $p_1\in [0,1]$ as defined in step~\ref{itm:rand_pred} of $\RSOA$. Thus, the learner's expected accumulated loss on $S$ is
\begin{equation} \label{eq:lem_upper_bound_def}
\M(\RSOA; S) = \lvert p_1 - y_1\rvert + \M(\RSOA; S').   
\end{equation}
By the induction hypothesis we have
\begin{equation} \label{eq:lem_upper_bound_induction}
 \M(\RSOA; S') \leq \RL\mleft(\cH_{x_1\to y_1}\mright).   
\end{equation}
Also, by Lemma~\ref{lem:rand_pred} it holds that $p_1 + \RL\mleft(\cH_{x_1\to 0}\mright) \leq \RL(\cH)$ and $1 - p_1 + \RL\mleft(\cH_{x_1\to 1}\mright) \leq \RL(\cH)$, which is equivalent to
\begin{equation} \label{eq:lem_upper_bound_dec_dim}
|p_1 - y_1| + \RL\mleft(\cH_{x_1\to y_1}\mright) \leq \RL(\cH). 
\end{equation}
Therefore, overall we get that
\begin{align*}
    \M (\RSOA ; S) &= |p_1 - y_1| + \M(\RSOA; S') \tag{Eq.~\eqref{eq:lem_upper_bound_def}} \\
                   &\leq |p_1 - y_1| + \RL\mleft(\cH_{x_1\to y_1}\mright) \tag{Eq.~\eqref{eq:lem_upper_bound_induction}} \\
                   &\leq \RL(\cH), \tag{Eq.~\eqref{eq:lem_upper_bound_dec_dim}} 
\end{align*}
as required.
\end{proof}

\section{Quasi-balanced Trees}
\subsection{Definition and Basic Properties}
\label{sec:quasi-balanced-trees}

The classical definition of the Littlestone dimension of a class $\cH$ is the maximum depth of a balanced (or complete) shattered tree. In contrast, the randomized Littlestone dimension is defined via quantifying over \emph{all} shattered trees. Further, in the deterministic case, balanced trees naturally describe optimal deterministic strategies for the adversary which force any learner to make a mistake on every example along a branch of the tree.

It is therefore natural to ask whether there is a type of shattered trees, analogous to balanced trees, which can be used to define the randomized Littlestone dimension.
In this subsection, we show that such an analog exists: a type of trees which we call \emph{quasi-balanced}; roughly speaking, quasi-balanced trees can be seen as a fractional relaxation of balanced trees. 
We further use this section to prove some useful properties of these trees, which will be used later on.

Informally, quasi-balanced trees are balanced under some weight function defined on the edges. To formally define quasi-balanced trees, we need to define \emph{weight functions} for trees.

Let $T$ be a non-empty tree with edge set $E$. Let $\cW = \cW(T)$ be the set of all functions $w \colon E \rightarrow [0,1]$, such that for every internal node with outgoing edges $e_0,e_1$ it holds that $w(e_0) + w(e_1) = 1$. Each function in $\cW$ is called a \emph{weight function} for $T$.

For every branch $b \in B(T)$ defined by a sequence of consecutive edges, define the \emph{weight of the branch $b$ with respect to $w$} by $w(b) = \sum_ {e \in b} w(e)$.

The expected weight of a random branch is always half the expected length of a random branch, as a simple inductive argument shows.

\begin{lemma} \label{lem:random-branch-weight}
For every non-empty tree $T$ and every weight function $w \in \cW(T)$, the expected weight of a random branch is $E_T/2$.
\end{lemma}


\begin{proof}
The proof is by induction on the depth of the tree. If $T$ is a leaf then the expected weight of a random branch is $0 = E_T/2$. If $T$ is not a leaf, let $e_0,e_1$ be the edges emanating from the root, and let $T_0,T_1$ be the corresponding subtrees. Applying the inductive hypothesis, the expected weight of a random branch in $T$ under $w$ is
\[
 \frac{w(e_0) + E_{T_0}/2}{2} + \frac{w(e_1) + E_{T_1}/2}{2} = \frac{1 + E_{T_0}/2 + E_{T_1}/2}{2} = E_T/2,
\]
using $w(e_0) + w(e_1) = 1$ and Equation~\eqref{eq:expected_branch_recursive}.
\end{proof}

This lemma prompts the following definition.

\begin{definition} \label{def:quasi-balanced}
A tree $T$ is \emph{quasi-balanced} if it is non-empty and there is a weight function $w \in \cW(T)$ under which all branches have weight $E_T/2$.
\end{definition}

We call $E_T/2$ the \emph{weight} of the tree, and denote it by $\lambda_T$.

\begin{lemma} \label{lem:quasi-balanced-uniqueness}
If a tree $T$ is quasi-balanced then there is a unique weight function $w$ under which all branches have the same weight. Explicitly, if $T'$ is a subtree of $T$ whose root is connected via edges $e_0,e_1$ to the subtrees $T_0,T_1$, then
\[
 w(e_0) = \frac{1 + \lambda_{T_1} - \lambda_{T_0}}{2} \text{ and } w(e_1) = \frac{1 + \lambda_{T_0} - \lambda_{T_1}}{2}.
\]
\end{lemma}
\begin{proof}
The trees $T',T_0,T_1$ are necessarily quasi-balanced, and in particular
\[
 w(e_0) + \lambda_{T_0} = w(e_1) + \lambda_{T_1}.
\]
Since $w(e_0) + w(e_1) = 1$, we can solve for $w(e_0),w(e_1)$, obtaining the claimed formula.
\end{proof}

Quasi-balanced trees are a generalization of balanced trees: every tree $T$ which is balanced is also quasi-balanced with weight $\lambda_T = d/2$, where $d$ is the depth of $T$. This weight is realized by the (unique) constant weight function that gives weight $1/2$ to all edges. The family of quasi-balanced trees is, however, much broader than the family of balanced trees.

\smallskip

Recall the definition of the randomized Littlestone dimension of the class $\cH$:
\[
\RL(\cH) =\frac{1}{2} \sup_{T \text{ shattered}} E_T.
\]
It turns out that in this definition, it suffices to take the supremum only over quasi-balanced trees. This will be easier to see through the characterization of quasi-balanced trees as \emph{monotone} trees.

\begin{definition}[Monotone Trees] \label{def:monotone}
A non-empty tree $T$ is \emph{weakly monotone} if 
\[E_{T} \geq \max\{E_{T_{0}}, E_{T_{1}}\},\] where $T_0$ and $T_1$ are the subtrees rooted at the children of the root of $T$.
A tree is \emph{monotone} if it is non-empty and all of its subtrees are weakly monotone.
\end{definition}


It is not hard to see that non-monotone trees need not be considered when computing the randomized Littlestone dimension.

\begin{lemma} \label{lem:randomized-littlestone-dimension-monotone}
For any non-empty hypothesis class $\cH$,
\[
\RL(\cH) =\frac{1}{2} \sup_{T \text{ shattered, monotone}} E_T.
\]
\end{lemma}
\begin{proof}
Consider a tree $T$ shattered by $\cH$ which is not monotone. Then there exists a vertex $v$ such that $E_{T_v} < E_{T_w}$, where $T_w$ is a tree rooted at a child of $v$. If we replace the subtree rooted at $v$ with the subtree $T_w$, we get a tree which is also shattered by $\cH$, and has higher expected branch length.

Repeating this process finitely many times, for each tree $T$ shattered by $\cH$ we obtain a monotone tree $T'$ shattered by $\cH$ satisfying $E_{T'} \ge E_T$, and the lemma follows. \end{proof}

The following theorem asserts that monotone and quasi-balanced trees are indeed equivalent.

\begin{theorem} \label{thm:quasi_equiv_mono}
A tree is quasi-balanced if and only if it is monotone.
\end{theorem}

\begin{corollary} \label{cor:randomized-littlestone-quasi-balanced}
For any non-empty hypothesis class $\cH$,
\[
\RL(\cH) =\frac{1}{2} \sup_{T \text{ shattered, quasi-balanced}} E_T.
\]
\end{corollary}

To prove Theorem~\ref{thm:quasi_equiv_mono}, we use the following simple observation.

\begin{observation} \label{obs:mono_equiv_balanced}
Let $T$ be a non-empty tree. Then $T$ is weakly monotone if and only if $|E_{T_{0}} -  E_{T_{1}}| \leq 2$, where $T_0$ and $T_1$ are the subtrees rooted at the children of the root of $T$.
\end{observation}

\begin{proof}
Equation~\eqref{eq:expected_branch_recursive} states that $2E_T = 2 + E_{T_0} + E_{T_1}$, and so $E_{T_0} \leq E_T$ is equivalent to $E_{T_0} - E_{T_1} \leq 2$. Similarly, $E_{T_1} \leq E_T$ is equivalent to $E_{T_1} - E_{T_0} \leq 2$. Hence $T$ is weakly monotone iff $|E_{T_0} - E_{T_1}| \leq 2$.
\end{proof}

\begin{proof}[Proof of Theorem~\ref{thm:quasi_equiv_mono}]
An empty tree is neither quasi-balanced nor monotone. Suppose therefore that we are given a non-empty tree $T$. We prove the equivalence by proving both implications separately.
\paragraph{Monotone $\implies$ Quasi-balanced.}
The proof is by induction on the depth of the tree. A tree of depth $0$ (the base case) is quasi-balanced with weight $E_T/2 = 0$. For the induction step, let $T_0,T_1$ be the subtrees rooted at the root's children. They are clearly monotone, and so by induction, there are weight functions $w_0 \in \cW(T_0)$ and $w_1 \in \cW(T_1)$ under which all branches in~$T_0$ have weight $\lambda_{T_0} = E_{T_0}/2$ and all branches in $T_1$ have weight $\lambda_{T_1} = E_{T_1}/2$.

Let $e_0,e_1$ be the edges connecting the root of $T$ to the roots of $T_0,T_1$, respectively. Define a weight function $w \in \cW(T)$ by defining $w(e) = w_0(e)$ if $e \in T_0$, $w(e) = w_1(e)$ if $e \in T_1$,
\[
 w(e_{0}) = \frac{1 + \lambda_{T_{1}} - \lambda_{T_{0}}}{2}, \quad \text{and}\quad w(e_{1}) = \frac{1 + \lambda_{T_{0}} - \lambda_{T_{1}}}{2}.
\]
Clearly $w(e_0) + w(e_1) = 1$. Observation~\ref{obs:mono_equiv_balanced} implies that $w(e_0),w(e_1) \in [0,1]$, and so indeed $w \in \cW(T)$. Since $w(e_0) + \lambda_{T_0} = w(e_1) + \lambda_{T_1}$, the weight function $w$ shows that $T$ is quasi-balanced.


\paragraph{Quasi-balanced $\implies$ Monotone.}
The proof is by induction on the depth of the tree. A tree of depth $0$ is monotone. For the induction step, we first observe that every proper subtree of $T$ is quasi-balanced, and so monotone by the inductive hypothesis. Hence it suffices to show that $T$ is weakly monotone.

Let $w$ be the unique weight function for $T$ under which each branch has weight $\lambda_T$. Let $e_0,e_1$ be the edges connecting the root of $T$ to the two subtrees $T_0,T_1$. According to Lemma~\ref{lem:quasi-balanced-uniqueness}, the weights of these edges are
\[
 w(e_0) = \frac{1 + \lambda_{T_1} - \lambda_{T_0}}{2} \text{ and }
 w(e_1) = \frac{1 + \lambda_{T_0} - \lambda_{T_1}}{2}.
\]
Since the weights are non-negative, we deduce that $|\lambda_{T_0} - \lambda_{T_1}| \leq 1$, and so $|E_{T_0} - E_{T_1}| \leq 2$. We conclude that $T$ is weakly monotone by Observation~\ref{obs:mono_equiv_balanced}.
\end{proof}






\subsection{A Concentration Lemma for Quasi-Balanced Trees} \label{sec:concentration}
Another interesting property of quasi-balanced trees is that the length of a random branch concentrates around its expectation. This property will be important for deriving tight bounds in Section~\ref{sec:expert-advice}.

\begin{proposition} [Concentration of branch lengths]\label{lem:quasi_concentration}
Let $T$ be a quasi-balanced tree, and let $X$ be the length of a random branch. Then for any $\epsilon > 0$, \[
 \Pr[X < (1-\epsilon) E_T] \leq \exp (-\epsilon^2 E_T/4) \quad \text{and} \quad
 \Pr[X > (1+\epsilon) E_T] \leq \exp (-\epsilon^2 E_T/4(1+\epsilon)).
\]
\end{proposition}

\begin{proof}
If $T$ is a single leaf then the result trivially holds since there is a single random branch. Therefore we can assume that $T$ is not a single leaf, and in particular, $E_T \geq 1$.

Let $b_0,b_1,b_2,\ldots$ be an infinite sequence of random coin tosses. We can choose a random branch of $T$ as follows. Let $v_0$ be the root of $T$. For $i \in \mathbb{N}$, if $v_i$ is not a leaf, then $v_{i+1}$ is obtained by following the edge labeled $b_i$. Otherwise, we define $v_{i+1} = v_i$. The resulting random branch has exactly the same distribution that we have been considering so far.

Let $L_i$ be the expected length of the branch given $b_0,\ldots,b_{i-1}$. This is an exposure martingale, as defined in Section~\ref{sec:prelim}.

In order to apply Azuma's inequality, we need to bound the random difference $|L_i - L_{i+1}|$. If $v_i$ is a leaf, then $L_{i+1} = L_i$. Otherwise, let $T'$ be the subtree rooted at $v_i$, and let $T'_0,T'_1$ be the subtrees rooted at the children of $v_i$. Thus $L_{i+1}$ is either $\lambda_0 := i + 1 + E_{T'_0}$ or $\lambda_1 := i + 1 + E_{T'_1}$, depending on the value of $b_i$. Moreover, $L_i = (\lambda_0 + \lambda_1)/2$ is the average of these two values.

Theorem~\ref{thm:quasi_equiv_mono} shows that $T'$ is weakly monotone, and so Observation~\ref{obs:mono_equiv_balanced} shows that $|E_{T'_0} - E_{T'_1}| \leq 2$. Consequently,
\[
 |L_i - L_{i+1}| = \frac{1}{2} |\lambda_0 - \lambda_1| \leq 1.
\]

The definition of $L_i$ implies that $L_\beta = X$ for all $\beta \geq X$. In particular, if $X < (1-\epsilon) E_T$ then $L_{\lceil E_T \rceil} < (1-\epsilon) E_T$. Applying Azuma's inequality and using $L_0 = E_T$, it follows that
\[
 \Pr[X < (1-\epsilon) E_T] \leq
 \Pr[L_{\lceil E_T \rceil} - E_T < -\epsilon E_T] \leq \exp \mleft(\frac{-\epsilon ^2 E_T^2 }{2 \lceil E_T \rceil} \mright) \leq \exp (-\epsilon^2 E_T/4),
\]
where the final inequality uses $\lceil E_T \rceil \leq E_T + 1 \leq 2E_T$.

The definition of $L_i$ also implies that $L_\beta \geq \beta$ whenever $\beta \leq X$.
In particular, if $X > (1+\epsilon) E_T$ then $L_{\lceil (1+\epsilon) E_T \rceil} \geq \lceil (1+\epsilon) E_T \rceil$. Therefore
\[
 \Pr[X > (1+\epsilon) E_T] \leq
 \Pr[L_{\lceil (1+\epsilon) E_T \rceil} - E_T > \epsilon E_T] \leq \exp \mleft(\frac{-\epsilon ^2 E_T^2 }{2 \lceil (1+\epsilon) E_T \rceil} \mright) \leq \exp (-\epsilon^2 E_T/4(1+\epsilon)),
\]
using $(1+\epsilon) E_T \geq 1$ as before.
\end{proof}

\section{Bounded Horizon}
\label{sec:bounded-horizon}

So far we have not put any restrictions on the number of rounds. However, in many circumstances we are interested in the online learning game when the number of rounds is bounded. We model this by assuming that the learner knows an upper bound on the number of rounds.
We define $\M^\star(\cH,\rounds)$ to be the optimal randomized mistake bound when the number of rounds is at most $\rounds$.


We can generalize Theorem~\ref{thm:characterization} to this setting. The required notion of randomized Littlestone dimension is
\[
 \RL(\cH,\rounds) = \frac{1}{2} \sup_{\substack{T \text{ shattered} \\ \depth(T) \leq \rounds}} E_T.
\]

The bounded randomized Littlestone dimension gives the precise mistake bound in this setting.

\begin{theorem}[Optimal Randomized Mistake Bound with Finite Horizon]\label{thm:characterization-rounds}
Let $\cH$ be an hypothesis class, and let $\rounds \in \mathbb{N}$. Then,
\[
\M^{\star}(\cH, \rounds) = \RL(\cH, \rounds).
\]
\end{theorem}

We prove Theorem~\ref{thm:characterization-rounds} in Section~\ref{sec:finite-horizon}.  This theorem immediately suggests the following questions:
\begin{enumerate}
    \item How many rounds are needed in order for the adversary to guarantee that the loss of the learner is at least $\RL(\cH) - \epsilon$? \\
    We prove in Section~\ref{sec:shallow-trees} that $2\RL(\cH) + O(\sqrt{\RL(\cH)}\log(\RL(\cH)/\epsilon))$ rounds always suffice, and $O(\log (1/\epsilon))$ rounds suffice as long as $\epsilon$ is small enough.
    \item What can we say about the loss of the learner when there are fewer than $2\RL(\cH)$ rounds? \\
    A trivial upper bound on $\RL(\cH,\rounds)$ is $\rounds/2$. In Section~\ref{sec:shallow-trees-few-rounds} we show that this bound is nearly optimal when $\rounds \leq 2\RL(\cH)$.
\end{enumerate}

The proofs of these results use concentration bounds on the depth of quasi-balanced trees, which we prove in Section~\ref{sec:concentration}.

\subsection{Proof of Theorem~\ref{thm:characterization-rounds}}
\label{sec:finite-horizon}

In this section we indicate how to generalize the proof of Theorem~\ref{thm:characterization} to the finite horizon setting, proving Theorem~\ref{thm:characterization-rounds}.

\smallskip

The lower bound $\RL(\cH,\rounds) \leq \M^\star(\cH,\rounds)$ follows directly from the statement of Lemma~\ref{lem:lower_bound_help}, since the length of $S$ is at most $\depth(T)$.

\smallskip

For the upper bound, we use a straightforward modification of algorithm $\RSOA$, which appears in Figure~\ref{fig:BRSOA}.

\begin{figure}
    \centering
    \begin{tcolorbox}
    \begin{center}
        $\BRSOA$: \textsc{Bounded Randomized $\SOA$}
    \end{center}
    \textbf{Input:} An hypothesis class $\cH$ and number of rounds $\rounds$.
    \\
    \textbf{Initialize:} Let $V^{(1)} = \cH$.\\
    \\
    \textbf{For $i=1,2,\dots,\rounds$}
    \begin{enumerate}
        \item Receive $x_i$. 
        \item Predict $p_i \in [0,1]$ such that the value 
        \[
        \max \mleft \{p_i + \RL\mleft(V^{(i)}_{x_i \to 0}, \rounds - i\mright), 1 - p_i + \RL\mleft(V^{(i)}_{x_i \to 1}, \rounds - i\mright) \mright \}   
        \]
        is minimized, where $V^{(i)}_{x_i\to b} = \{h\in V^{(i)} : h(x_i)=b\}$.
        \item Receive true label $y_i$.
        \item Update $V^{(i+1)} = V^{(i)}_{x_i \to y_i}$.
    \end{enumerate}
    \end{tcolorbox}
    \caption{$\BRSOA$ is a bounded variant of $\RSOA$.}
    \label{fig:BRSOA}
\end{figure}

We start by extending Observation~\ref{obs:expected_prop}: if $\cH$ is a non-empty hypothesis class and $\rounds > 0$ then
\[
 \RL(\cH,\rounds) = \frac{1}{2} \max_{x \in \cX} \bigl(1 + \RL(\cH_{x \to 0}, \rounds-1) + \RL(\cH_{x \to 1}, \rounds-1) \bigr).
\]
The proof is identical. Since there are only finitely many unlabeled trees of depth at most $\rounds$, we can replace the supremum with a maximum.

The next step is to generalize Lemma~\ref{lem:rand_pred}, which now states that for every hypothesis class $\cH$, instance $x \in \cX$, and $\rounds > 0$, there exists $p \in [0,1]$ so that
\[
 p + \RL\mleft(\cH_{x\to 0}, \rounds-1\mright) \leq \RL(\cH, \rounds)
 \quad \text{and} \quad
 (1 - p) + \RL\mleft(\cH_{x\to 1}, \rounds-1\mright) \leq \RL(\cH, \rounds).
 \]
The proof is identical, using the generalized Observation~\ref{obs:expected_prop}.

Finally, we prove the following generalization of Lemma~\ref{lem:characterization_upper}: for every hypothesis class $\cH$, any parameter $\rounds$, and any realizable input sequence $S$ of length at most $\rounds$,
\[
 \M(\BRSOA; S) \leq \RL(\cH, \rounds).
\]
The proof is identical, using the generalized Lemma~\ref{lem:rand_pred}.


\subsection{Approaching \texorpdfstring{$\RL(\cH)$}{RL(H)}} \label{sec:shallow-trees}
As a simple consequence of the concentration bound proved in Proposition~\ref{lem:quasi_concentration}, we show that we can approach $\RL(\cH)$ using relatively shallow trees, quantified as follows.

\begin{proposition} \label{pro:shallow-tree}
Let $\cH$ be a non-empty hypothesis class with finite randomized Littlestone dimension $\RL(\cH)$.

For every $\epsilon > 0$ there is a tree $T$ shattered by $\cH$ satisfying $E_T/2 \geq \RL(\cH) - \epsilon$ whose depth is at most
\[
 2\RL(\cH) + O\left(\sqrt{\RL(\cH) \log \frac{\RL(\cH)}{\epsilon}} + \log \frac{1}{\epsilon}\right) = 2\RL(\cH) + O\left(\sqrt{\RL(\cH)} \log \frac{\RL(\cH)}{\epsilon}\right).
\]
\end{proposition}

\iffull
\else
This means that the adversary can force the learner to suffer a loss of $\RL(\cH) - \epsilon$ after only $2\RL(\cH) + O(\sqrt{\RL(\cH)\log(\RL(\cH)/\epsilon)} + \log(1/\epsilon))$ rounds.
\fi
In contrast, at least $2\RL(\cH) - 2\epsilon$ rounds are clearly needed, since a learner who predicts $1/2$ at each round suffers a loss of $R/2$ after $R$ rounds.

\smallskip

We prove Proposition~\ref{pro:shallow-tree} via the following technical estimate.

\begin{lemma}
\label{lem:long-paths}
Let $T$ be a monotone tree, and let $T^{\leq k}$ result from truncating it to the first $k$ levels (all branches in $T^{\leq k}$ have length at most $k$). If $k \geq E_T$ then
\[
 E_{T^{\leq k}} \geq E_T - 15 \sqrt{E_T} \exp \left(- \frac{(k-E_T)^2}{8E_T} \right) - 10 \exp \left( - \frac{k - E_T}{8} \right).
\]
\end{lemma}
\begin{proof}
Let $X$ be the length of a random branch of $T$.
Using $X$, we can express the difference between $E_{T^{\leq k}}$ and $E_T$ explicitly:
\[
 E_T - E_{T^{\leq k}} =
 \sum_{t = k}^\infty \Pr[X > t].
\]
Applying Proposition~\ref{lem:quasi_concentration}, we deduce that the difference is at most
\[
 \Delta := \sum_{t = k}^\infty \exp \left( - \frac{(t - E_T)^2}{4t} \right) \leq \exp \left(- \frac{(k-E_T)^2}{4k} \right) + \int_k^\infty \exp \left( - \frac{(t - E_T)^2}{4t} \right) \, dt.
\]


If $k \leq 2E_T$ then
\begin{align*}
 \Delta &\leq \exp \left(- \frac{(k-E_T)^2}{4k} \right) +
 \int_k^{2E_T} \exp \left( - \frac{(t - E_T)^2}{8E_T} \right) \, dt + \int_{2E_T}^\infty \exp \left( - \frac{t - E_T}{8} \right) \, dt \\ &\leq 5\sqrt{E_T} \exp \left(- \frac{(k-E_T)^2}{8E_T} \right) + 8 \exp \left( - \frac{E_T}{8} \right) \leq 15 \sqrt{E_T} \exp \left(- \frac{(k-E_T)^2}{8E_T} \right),
\end{align*}
using the well-known Gaussian tail bound,
\[
 \int_k^\infty e^{-(t - \mu)^2/2\sigma^2} \, dt = \sqrt{2\pi\sigma^2} \Pr[N(\mu,\sigma) > k] \leq \sqrt{2\pi\sigma^2} e^{-(k-\mu)^2/2\sigma^2} \quad (k \ge \mu),
\]
to bound the first integral.

If $k \geq 2E_T$ then
\[
 \Delta \leq \exp \left( - \frac{k - E_T}{8} \right) + \int_k^\infty \exp \left( - \frac{t - E_T}{8} \right) \, dt \leq 9\exp \left( - \frac{k - E_T}{8} \right). \qedhere
\]
\end{proof}



We can now prove Proposition~\ref{pro:shallow-tree}.

\begin{proof}[Proof of Proposition~\ref{pro:shallow-tree}]
Applying Lemma~\ref{lem:randomized-littlestone-dimension-monotone}, we can find a monotone tree $T$ shattered by $\cH$ such that $2\RL(\cH) - \epsilon/2 \leq E_T \leq 2\RL(\cH)$. Let
\[
 k = E_T + \sqrt{8E_T \log \frac{60\sqrt{E_T}}{\epsilon}} + 8\log \frac{20}{\epsilon} = E_T + O\left(\sqrt{E_T \log \frac{E_T}{\epsilon}} + \log \frac{1}{\epsilon}\right).
\]
Lemma~\ref{lem:long-paths} implies that $E_{T^{\leq k}} \geq E_T - \epsilon/2$, and so $E_{T^{\leq k}} \geq 2\RL(\cH) - \epsilon$.
\end{proof}



\subsection{Mistake Bound for Few Rounds} \label{sec:shallow-trees-few-rounds}

Another truncation argument allows us to estimate $\RL(\cH,\rounds)$ for small $\rounds$.

\begin{proposition} \label{pro:truncation-few-rounds}
Let $\cH$ be a non-empty hypothesis class with finite randomized Littlestone dimension $\RL(\cH)$.

If $\rounds \leq \RL(\cH)$ then $\RL(\cH,\rounds) = \rounds/2$.

If $\rounds \leq 2\RL(\cH)$ then
\[
 \frac{\rounds}{2} - O(\sqrt{\rounds \log \rounds}) \leq \RL(\cH,\rounds) \leq \frac{\rounds}{2}.
\]

Furthermore, if $\rounds \leq 2\RL(\cH) - \sqrt{8\RL(\cH) \ln \RL(\cH)}$ then
\[
 \frac{\rounds}{2} - 1 < \RL(\cH,\rounds) \leq \frac{\rounds}{2},
\]

and if $\rounds \geq 2\RL(\cH) - \sqrt{8\RL(\cH) \ln \RL(\cH)}$ then
\[
\RL(\cH) - O\mleft(\sqrt{\RL(\cH) \log \RL(\cH)}\mright) \leq \RL(\cH,\rounds) \leq \RL(\cH).
\]
\end{proposition}
\begin{proof}
A learner that always predicts $1/2$ suffers a loss of exactly $1/2$ each round, showing that $\RL(\cH,\rounds) \leq \rounds/2$ for each $\rounds$.
In contrast, if $T$ is a tree shattered by $\cH$ then Theorem~\ref{thm:characterization-rounds} shows that $\RL(\cH,\rounds) \geq E_{T^{\leq \rounds}}/2$, and we will use this to give lower bounds on $\RL(\cH,\rounds)$.

Suppose first that $\rounds \leq \RL(\cH)$.
\iffull
\else
Our characterization shows that $m_T \geq E_T/2$.
\fi
If $E_T$ is close enough to $2\RL(\cH)$ then $m_T \geq \RL(\cH)$ (since $m_T$ is an integer), and so $T^{\leq \rounds}$ is a complete tree of depth $\rounds$. This shows that $\RL(\cH,\rounds) \geq \rounds/2$.

In order to prove the remaining results, suppose that $\rounds \leq 2\RL(\cH)$, and consider a tree $T$ shattered by $\cH$ satisfying $E_T = 2\RL(\cH) - \delta \geq \rounds$. Proposition~\ref{lem:quasi_concentration} shows that a random branch of $T^{\leq \rounds}$ has depth $\rounds$ with probability at least $1 - \exp\bigl(-\frac{(E_T - \rounds)^2}{4E_T}\bigr)$, and so
\[
 \RL(\cH,\rounds) \geq
 \left(1 - \exp \left(- \frac{(E_T - \rounds)^2}{4E_T} \right) \right) \cdot \frac{\rounds}{2} \longrightarrow
 \left(1 - \exp \left(- \frac{(2\RL(\cH) - \rounds)^2}{8\RL(\cH)} \right) \right) \cdot \frac{\rounds}{2},
\]
where the limit is taken along a sequence of trees shattered by $\cH$ and satisfying $E_T \to 2\RL(\cH)$.

If $\rounds \leq \rounds_0 := 2\RL(\cH) - \sqrt{8\RL(\cH) \ln \RL(\cH)}$ then this gives
\[
 \RL(\cH,\rounds) \geq \left(1 - \frac{1}{\RL(\cH)} \right) \cdot \frac{\rounds}{2} > \frac{\rounds}{2} - 1.
\]
If $\rounds_0 \leq \rounds \leq 2\RL(\cH)$ then
\[
 \RL(\cH,\rounds) \geq
 \RL(\cH,\rounds_0) \geq \frac{\rounds_0-2}{2} \geq
 \frac{\rounds-2}{2} - \sqrt{2\RL(\cH) \ln \RL(\cH)} \geq
 \frac{\rounds}{2} - O(\sqrt{\rounds \log \rounds}).
\]
Finally, if we only assume that $\rounds \geq \rounds_0$ then
\[
 \RL(\cH,\rounds) \geq
 \RL(\cH,\rounds_0) \geq
 \frac{\rounds_0-2}{2} \geq
 \RL(\cH) - \sqrt{2\RL(\cH) \ln \RL(\cH)} - 1.\qedhere
\]
\end{proof}

\section{Mistake Bounds in the \texorpdfstring{$k$}{k}-Realizable Setting} \label{sec:k-realizable}

So far we have considered online learning when the adversary is restricted to choose labels which are consistent with one of the hypotheses in the hypothesis class, a setting known as the \emph{realizable} setting.
This is a quite restrictive assumption, and there are many ways to relax it.

In this section we concentrate on the \emph{$k$-realizable} setting, in which the answers of the adversary are consistent with one of the hypotheses in the class \emph{up to at most $k$ mistakes}.
Our goal is to characterize the optimal mistake bounds in this setting, for both deterministic and randomized learners, generalizing Theorems~\ref{thm:Littlestone} and \ref{thm:characterization}.
Our characterizations are based on \emph{$k$-shattered trees}, in which each branch is consistent with one of the hypotheses in the class up to at most $k$ mistakes.

\smallskip

If all instances in a sequence of examples are distinct, then the sequence is $k$-realizable by $\cH$ if and only if it is realizable by the \emph{$k$-expansion} of $\cH$, consisting of all hypotheses $h'$ which disagree with some hypothesis $h \in \cH$ on at most $k$ instances. However, this need not be the case. For example, the sequence $(x,0),(x,1)$ is $1$-realizable by the hypothesis class $\cH$ consisting of all constant functions.



Nevertheless, the arguments in this section are very similar to their counterparts in the realizable setting.

\subsection{Weighted Hypothesis Classes} \label{sec:weighted-hypothesis-class}

While we are interested mainly in the $k$-realizable setting, we consider a more general setting in which the number of allowed mistakes can depend on the hypothesis. This will be useful in the subsequent proofs.

A weighted hypothesis class $\cW$ is a collection of pairs $(h,w)$, where $h\colon \cX \to \cY$ is an hypothesis and $w \in \mathbb{N}$ is the allowed number of mistakes (possibly zero). Furthermore, all hypotheses are distinct (that is, $\cW$ cannot contain two different pairs $(h,w_1),(h,w_2)$). An input sequence $(x_1,y_1),\dots,(x_t,y_t)$ is \emph{realizable} by a weighted hypothesis class $\cW$ if there exists $(h,w) \in \cW$ such that $h(x_i) \neq y_i$ for at most $w$ many examples in the sequence. A tree is \emph{shattered} by $\cW$ if each of its branches is realized by $\cW$.

Given an hypothesis class $\cH$, a learning rule which observes the labeled example $(x,y)$ can restrict itself to $\cH_{x \to y} = \{ h \in \cH : h(x) = y \}$. The corresponding operation for weighted hypothesis classes is
\[
 \cW_{x \to y} = \{ (h,w) : (h,w) \in \cW, h(x) = y \} \cup \{ (h,w-1) : (h,w) \in \cW, h(x) \neq y, w > 0 \}.
\]
In words, we decrease the allowed number of mistakes for each hypothesis inconsistent with the given example $(x,y)$, removing hypotheses which has zero mistakes left.

\smallskip

For every weighted hypothesis class $\cW$, we define its Littlestone dimension and its randomized Littlestone dimension by
\[
\LD(\cW) = \sup_{T \text{ shattered}} m_T
\quad \text{and} \quad
\RL(\cW) = \frac{1}{2} \sup _{T \text{ shattered}} E_T,
\]
where the supremum is taken over all trees shattered by $\cW$. As in the realizable setting, we define $\LD(\emptyset) = \RL(\emptyset) = -1$ for convenience.

Our main results in this section extend Theorems~\ref{thm:Littlestone} and~\ref{thm:characterization} to this more general setting.

\begin{theorem}[Optimal Deterministic Mistake Bound]\label{thm:characterization_det_weighted}
Let $\cW$ be a weighted hypothesis class. Then,
\[
\M^{\star}_D(\cW) = \LD(\cW). 
\]
\end{theorem}

\begin{theorem}[Optimal Randomized Mistake Bound]\label{thm:characterization_rand_weighted}
Let $\cW$ be a weighted hypothesis class. Then,
\[
\M^{\star}(\cW) = \RL(\cW).
\]
\end{theorem}

We prove these theorems in the following subsections, making use of the following fundamental observation, which follows directly from the definitions:

\begin{observation} \label{obs:weighted-decomposition}
Let $\cW$ be a weighted hypothesis class. The sequence $(x_1,y_1),\ldots,(x_t,y_t)$ is realizable by $\cW$ iff the sequence $(x_2,y_2),\ldots,(x_t,y_t)$ is realizable by $\cW_{x_1 \to y_1}$.

Similarly, let $T$ is a tree whose root is labeled by $x$, and let $T_0,T_1$ be the subtrees rooted at the children of the root. Then $T$ is realizable by $\cW$ iff $T_0$ is realizable by $\cW_{x \to 0}$ and $T_1$ is realizable by $\cW_{x \to 1}$.
\end{observation}




\paragraph{The $k$-realizable setting.}

Let $\cH$ be an hypothesis class, and let $k \in \mathbb{N}$. A sequence of examples $S = \{(x_i,y_i)\}_{i=1}^t$ is \emph{$k$-realizable} by $\cH$ if there exists $h \in \cH$ such that $h(x_i) \neq y_i$ for at most $k$ indices $i$. We denote the corresponding mistake bounds by $\M^\star(\cH,k),\M^\star_D(\cH,k)$. These are defined just as in the realizable setting, the only difference being that the sequence of examples provided by the adversary need only be $k$-realizable by $\cH$.

We say that a tree is \emph{$k$-shattered by $\cH$} if every branch is $k$-realizable by $\cH$.
The corresponding deterministic and randomized $k$-Littlestone dimension of a class $\cH$ are
\[
\LD_k(\cH) = \sup_{T  \text{ } k\text{-shattered}} m_T
\quad \text{and} \quad
\RL_k(\cH) =\frac{1}{2} \sup_{T  \text{ } k\text{-shattered}} E_T.
\]

If we define $\cW_{\cH,k} = \{ (h,k) : h \in \cH \}$, then a sequence of examples is $k$-realizable by $\cH$ if it is realizable by $\cW_{\cH,k}$. In other words, the $k$-realizable setting is a special case of weighted hypothesis classes, where all weights are equal to~$k$. Therefore we immediately conclude the following theorems, by applying the preceding theorems to $\cW_{\cH,k}$:

\begin{theorem}[Optimal Deterministic Mistake Bound]\label{thm:characterization_det_k}
Let $\cH$ be an hypothesis class, and let $k \in \mathbb{N}$. Then,
\[
\M^{\star}_D(\cH, k) = \LD_k(\cH). 
\]
\end{theorem}

\begin{theorem}[Optimal Randomized Mistake Bound]\label{thm:characterization_rand_k}
Let $\cH$ be an hypothesis class, and let $k \in \mathbb{N}$. Then,
\[
\M^{\star}(\cH, k) = \RL_k(\cH).
\]
\end{theorem}

\smallskip
Using the classic lower bounds of \cite{littlestone1994weighted,bendavid2009agnostic} and recent results of~\cite{alon2021adversarial}, we can bound the optimal mistake bound in terms of the \emph{realizable} Littlestone dimension:

\begin{theorem} \label{thm:k-realizable-littlestone-bound}
Let $\cH$ be an hypothesis class with at least two hypotheses, and let $k \in \mathbb{N}$. Then,
\[
 \M^\star(\cH, k) = k + \Theta\!\left(\sqrt{k \cdot \LD(\cH)} + \LD(\cH)\right).
\]
\end{theorem}
We prove this result in Section~\ref{sec:k-realizable-littlestone-bound}.
Note that since $\LD(\cH)$ and $\RL(\cH)$ differ by at most a constant factor, the theorem still holds if we replace $\LD(\cH)$ by $\RL(\cH)$.

\smallskip

Using the experts algorithm of~\cite{koolen:15}, we can construct an algorithm which works in the adaptive setting, that is, without knowledge of $k$:

\begin{theorem}
\label{thm:adaptive-algorithm} 
Let $\cH$ be an hypothesis class.
There is an algorithm $\Squint$ such that for every input sequence $S$ which is $k^\star$-realizable by $\cH$,
\[
 \M(\Squint;S) \leq \M^\star(\cH,k^\star) + O\!\left( \sqrt{ \M^\star(\cH, k) \log ((k^\star+1) \log \M^\star(\cH,k^\star))} \right).
\]
Furthermore, $\Squint$ is adaptive, that is, it has no knowledge of $k^\star$.
\end{theorem}
We describe and analyze the algorithm in Section~\ref{sec:adaptive-algorithm}.




\subsection{Proof of Optimal Deterministic Mistake Bound}\label{sec:krealdetproof}

The case $\cW = \emptyset$ holds by definition. Therefore we assume that $\cW \neq \emptyset$. The lower bound ``$\LD(\cW) \leq \M^\star_D(\cW)$'' boils down to the following lemma:

\begin{lemma} \label{lem:lower-bound-weighted-help}
Let $\cW$ be a weighted hypothesis class, and let $T$ be a finite tree which is shattered by $\cW$. Then, for every deterministic learning rule $\Lrn$ there exists a realizable sequence $S$ so that $\M(\Lrn;S) \geq m_T$. Furthermore, $S$ corresponds to one of the branches in $T$.
\end{lemma}
\begin{proof}
We construct the sequence $S$ by traversing $T$, starting at the root $v_1$. At step $i$, we send $\Lrn$ the instance $x_i$ labelling $v_i$. If the learner predicts $\hat{y}_i$, we set the true label to $y_i = 1-\hat{y}_i$, and let $v_{i+1}$ be the vertex obtained from $v_i$ by following the leaf labeled $y_i$. We stop once the process reaches a leaf.

By construction, $S$ corresponds to one of the branches of $T$, and the number of mistakes is $|S| \geq m_T$. Since $T$ is shattered by $\cW$, then $S$ is realizable by $\cW$.
\end{proof}

By applying the lemma on every shattered tree and taking the supremum, we conclude the lower bound:

\begin{corollary}[Lower bound] \label{cor:characterization_weighted_upper}
For every weighted hypothesis class $\cW$ it holds that $\M^\star_D(\cW) \geq \LD(\cW)$.
\end{corollary}

We now turn to prove the upper bound ``$\LD(\cW) \geq \M^\star_D(\cW)$''. This is achieved via the $\WSOA$ learning rule, depicted in Figure~\ref{fig:WSOA-det}.

\begin{figure}
    \centering
    \begin{tcolorbox}
    \begin{center}
        \textsc{$\WSOA$}
    \end{center}
    \textbf{Input:} A weighted hypothesis class $\cW$.
    \\
    \textbf{Initialize:} Let $V^{(1)} = \cW$.\\
    \\
    \textbf{for $i=1,2,\dots$}
    \begin{enumerate}
        \item Receive $x_i$.
        \item Predict
            \[
            \hat{y}_i = \argmax_{b\in \cY} \LD\mleft(V^{(i)}_{x_i \to b}\mright).
            \]
        \item Receive true label $y_i$.
        \item Update $V^{(i+1)} = V^{(i)}_{x_i \to y_i}$.
    \end{enumerate}
    \end{tcolorbox}
    \caption{The weighted version of $\SOA$.} 
    \label{fig:WSOA-det}
\end{figure}

\begin{lemma}[Upper bound] \label{lem:k_det_upper}
Let $\cW$ be a non-empty weighted hypothesis class. 
The $\WSOA$ learner described in Figure~\ref{fig:WSOA-det} has the mistake bound
\[
\M(\WSOA ; S) \leq \LD(\cW)
\]
for every input sequence $S$ realizable by $\cW$.
\end{lemma}

\begin{proof}
We will show that each time that $\WSOA$ makes a mistake, the Littlestone dimension drops by at least $1$. That is, if $\hat{y}_i \neq y_i$ then $\LD(V^{(i+1)}) < \LD(V^{(i)})$. Since the Littlestone dimension is always non-negative, it follows that $\WSOA$ makes at most $\LD(\cW)$ mistakes.

Suppose that $\hat{y}_i \neq y_i$ yet $\LD(V^{(i+1)}) = \LD(V^{(i)})$. The choice of $\hat{y}_i$ shows that $\LD(V^{(i)}_{x_i \to 0}) = \LD(V^{(i)}_{x_i \to 1}) = \LD(V^{(i)})$. This is, however, impossible. Indeed, take trees $T_0,T_1$ shattering $V^{(i)}_{x_i \to 0},V^{(i)}_{x_i \to 1}$ with $m_{T_0} = m_{T_1} = \LD(V^{(i)})$. Observation~\ref{obs:weighted-decomposition} shows that the tree $T$ whose root is labeled $x_i$ and in which $T_0,T_1$ are the subtrees of the root's children is shattered by $V^{(i)}$. Since $m_T = \LD(V^{(i)}) + 1$, we reach a contradiction.
\end{proof}

\subsection{Proof of Optimal Randomized Mistake Bound}\label{sec:krealrandproof}

The proof of the optimal mistake bound in the randomized setting, Theorem~\ref{thm:characterization_rand_weighted}, is very similar to the proof of its counterpart in the realizable setting, Theorem~\ref{thm:characterization}.

The proof of the lower bound ``$\RL(\cW) \leq \M^\star(\cW)$'' is virtually identical to the proof of Lemma~\ref{lem:lower_bound_help}.

The proof of the upper bound ``$\RL(\cW) \geq \M^\star(\cW)$'' uses $\WRSOA$, the weighted counterpart of $\RSOA$, which appears in Figure~\ref{fig:WSOA-rand}. The proof of Lemma~\ref{lem:characterization_upper} extends, with virtually no changes, to show that $\M(\WRSOA;S) \leq \RL(\cW)$ for every input sequence $S$ realizable by $\cW$.

\begin{figure}
    \centering
    \begin{tcolorbox}
    \begin{center}
        \textsc{$\WRSOA$}
    \end{center}
    \textbf{Input:} A weighted hypothesis class $\cW$.
    \\
    \textbf{Initialize:} Let $V^{(1)} = \cW$.\\
    \\
    \textbf{for $i=1,2,\dots$}
    \begin{enumerate}
        \item Receive $x_i$.
        \item Predict $p_i \in [0,1]$ such that the value 
                \[
                \max \mleft \{p_i + \RL\mleft(V^{(i)}_{x_i \to 0}\mright), 1 - p_i + \RL\mleft(V^{(i)}_{x_i \to 1}\mright) \mright \}   
                \]
        is minimized.
        \item Receive true label $y_i$.
        \item Update $V^{(i+1)} = V^{(i)}_{x_i \to y_i}$.
    \end{enumerate}
    \end{tcolorbox}
    \caption{The weighted version of $\RSOA$.} 
    \label{fig:WSOA-rand}
\end{figure}

\subsection{Explicit Bounds in Terms of Littlestone Dimension}
\label{sec:k-realizable-littlestone-bound}

Here we prove Theorem~\ref{thm:k-realizable-littlestone-bound}, which bounds $\M^\star(\cH, k)$ in terms of $k$ and $\LD(\cH)$ (or $\RL(\cH)$).
In the proof, we use the notation $\M^\star(\cH, k, \rounds)$ for the optimal mistake bound in the $k$-realizable setting when the number of rounds is bounded by $\rounds$, and the notation $\M^\star_{\Agn}(\cH, \rounds)$ for the optimal mistake bound when the number of rounds is bounded by $\rounds$, but there is no limitation on the number of mistakes made by the best hypothesis in $\cH$.

We first prove the upper bound.
\cite{alon2021adversarial} have shown that, for any time horizon $\rounds$, we always have 
$\M^{\star}(\cH, k, \rounds) \leq k + O\!\left( \sqrt{ \rounds \cdot \LD(\cH) } \right)$.
By extending Proposition~\ref{pro:shallow-tree} to the $k$-realizable case, time horizon $\rounds = O(\M^{\star}(\cH, k))$ suffices to guarantee 
$\M^{\star}(\cH, k) \leq \M^{\star}(\cH, k, \rounds)+1$.
Plugging this time horizon into the result of 
\cite{alon2021adversarial} reveals that 
\[
\M^{\star}(\cH, k) \leq k + O\!\left( \sqrt{ \M^{\star}(\cH, k) \cdot \LD(\cH) } \right).
\]
Solving this quadratic inequality in $\sqrt{\M^{\star}(\cH, k)}$ 
yields the upper bound claimed in the theorem.

We now turn to prove the lower bound. We  consider two cases. If $k \leq \LD(\cH)$, it suffices to prove a lower bound of $k + \Omega(\LD(\cH))$. The lower bound $k + \LD(\cH)/2$ of \cite{littlestone1994weighted} establishes that. In the complementary case, suppose that $k > \LD(\cH)$. Therefore, we only need to prove a lower bound of $k + \Omega \mleft(\sqrt{k \cdot \LD(\cH)} \mright)$. This follows from the following adaptation of the classic regret bound of \cite{bendavid2009agnostic}. They showed that $\M^\star_{\Agn}(\cH, \rounds) \geq b + \Omega \mleft(\sqrt{\LD(\cH) \cdot \rounds} \mright)$, where $b$ is the minimal number of mistakes made by a best hypothesis $h^\star \in \cH$. To adapt this bound to our setting, first play the game for $\rounds=k$ rounds, forcing a loss of at least $b + \Omega \mleft(\sqrt{k \cdot\LD(\cH)} \mright)$ on the learner. Now, as we prove in Theorem~\ref{thm:single-expert-intro},\footnote{We prove this result in the setting of prediction with expert advice, but it holds for general hypothesis classes (as long as the domain is non-empty).} the adversary can further force the learner a loss arbitrarily close to $k-b$, using an input sequence which is $(k-b)$-realizable by $h^\star$. Overall, the input sequence is $k$-realizable by $h^\star$, and we get the desired lower bound $k + \Omega \mleft(\sqrt{k \cdot \LD(\cH)} \mright)$.

\subsection{Adapting to \texorpdfstring{$k$}{k}}
\label{sec:adaptive-algorithm}

This section presents our proof of Theorem~\ref{thm:adaptive-algorithm}, showing that 
it is possible to adapt to the value of $k$ 
without sacrificing too significantly in the expected mistake bound.

The adaptive technique we propose uses an experts algorithm of 
\cite{koolen:15} named $\Squint$, with experts 
defined by the optimal randomized algorithm for the $k$-realizable 
setting, for all values of $k \in \mathbb{N}$ (including $k = 0$).

The experts algorithm $\Squint$ accepts an input sequence $S = (x_1,y_1), \ldots, (x_n,y_n)$ and a list of learners $\Lrn_k$, each with an associated weight $\pi_k$. The weights $\pi_k$ should form a probability distribution. With an appropriate choice of parameters, $\Squint$ has the following guarantee~\cite[Theorem 3]{koolen:15}:
\begin{equation} \label{eq:squint-guarantee}
\M(\Squint;S) \leq \min_k \left\{
\M(\Lrn_k;S) + O \! \left( \sqrt{V_k \log \frac{\log V_k}{\pi_k}} + \log \frac{1}{\pi_k} \right)
\right\},
\end{equation}
where $V_k$ is an uncentered variance term given by
\[
 V_k = \sum_{i=1}^n (|\Squint(x_1,y_1,\ldots,x_{i-1},y_{i-1},x_i) - y_i| - |\Lrn_k(x_1,y_1,\ldots,x_{i-1},y_{i-1},x_i) - y_i|)^2.
\]
Since both absolute values are in the range $[0,1]$, we have
\begin{align*}
 V_k &\leq \sum_{i=1}^n |\Squint(x_1,y_1,\ldots,x_{i-1},y_{i-1},x_i) - y_i| + \sum_{i=1}^n |\Lrn_k(x_1,y_1,\ldots,x_{i-1},y_{i-1},x_i) - y_i| \\ &= \M(\Squint;S) + \M(\Lrn_k;S).
\end{align*}
For any given $k$, if $\M(\Squint;S) > \M(\Lrn_k;S)$, 
then we have $V_k \leq 2 \M(\Squint;S)$, 
so that \eqref{eq:squint-guarantee} implies 
\begin{equation*}
\M(\Squint;S) \leq \M(\Lrn_k;S) + O\!\left( \sqrt{\M(\Squint;S) \log \frac{\log \M(\Squint;S)}{\pi_k}} + \log \frac{1}{\pi_k} \right).
\end{equation*}
This inequality trivially holds as well in the case 
$\M(\Squint;S) \leq \M(\Lrn_k;S)$ due to the first term on the right hand side.
Moreover, this inequality further implies 
\begin{equation*} 
\M(\Squint;S) = O\!\left( \M(\Lrn_k;S) + \log \frac{1}{\pi_k} + 1 \right).
\end{equation*}
To see this, note that were it not the case, 
we could upper bound each 
$\M(\Lrn_k;S)$ and $\log \frac{1}{\pi_k}$ 
on the right hand side by $\M(\Squint;S)/c$ for some 
large constant $c$, making the right hand side 
strictly less than $\M(\Squint;S)$: a contradiction.
Plugging in this upper bound on $\M(\Squint;S)$ 
into the $\log\log \M(\Squint;S)$
term and simplifying with elementary inequalities reveals 
\begin{equation*}
\M(\Squint;S) \leq \M(\Lrn_k;S) + O\!\left( \sqrt{\M(\Squint;S) \log \frac{\log \M(\Lrn_k;S)}{\pi_k}} + \log \frac{1}{\pi_k} \right).
\end{equation*}
This is a quadratic inequality in $\sqrt{\M(\Squint;S)}$. 
Solving the quadratic for the range of $\M(\Squint;S)$ where the inequality holds, we have
\begin{equation*}
\M(\Squint;S) \leq \M(\Lrn_k;S) + O\!\left( \sqrt{\M(\Lrn_k;S) \log \frac{\log \M(\Lrn_k;S)}{\pi_k}} + \log \frac{\log \M(\Lrn_k;S)}{\pi_k} \right).
\end{equation*}
Since this holds for any $k$, we conclude that 
\begin{equation} \label{eq:squint-guarantee-2}
 \M(\Squint;S) \leq \min_k \left\{
 \M(\Lrn_k;S) + O \! \left( \sqrt{\M(\Lrn_k;S) \log \frac{\log \M(\Lrn_k;S)}{\pi_k}} + \log \frac{\log \M(\Lrn_k;S)}{\pi_k} \right)
 \right\}.
\end{equation}

We instantiate $\Squint$ with algorithm $\WRSOA$ of Figure~\ref{fig:WSOA-rand}. Namely, for every $k$, we let $\Lrn_k$ be the instantiation of $\WRSOA$ with $\cW_{\cH,k}$. We use the weights $\pi_k = \frac{1}{(k+1)(k+2)}$. Since $\pi_k = \frac{1}{k+1} - \frac{1}{k+2}$, they indeed constitute a probability distribution. Since $\WRSOA$ achieves the optimal mistake bound (see Section~\ref{sec:krealrandproof}), Eq.~\eqref{eq:squint-guarantee-2} shows that if $S$ is $k^\star$-realizable by $\cH$ then
\[
 \M(\Squint;S) \leq
 \M^\star(\cH,k^\star) +
 O \! \left(
 \sqrt{\M^\star(\cH,k^\star) \log\bigl((k^\star+1) \log \M^\star(\cH,k^\star)\bigr)} + \log\bigl((k^\star+1) \M^\star(\cH,k^\star) \bigr)
 \right).
\]
Since $\M^\star(\cH,k^\star) \geq k^\star/2$, the term $\log\bigl((k^\star+1) \M^\star(\cH,k^\star)\bigr)$ can be swallowed by the preceding term.

\section{Prediction using Expert Advice} \label{sec:expert-advice}

In this section, we consider the problem of \emph{prediction using expert advice}, which was raised in~\cite{vovk1990aggregating,littlestone1994weighted}. Specifically, we consider the $k$-realizable setting, which was suggested in~\cite{cesa1996line,cesa1997use} and further studied in~\cite{abernethy2006continuous, mukherjee2010learning, branzei2019online}. 


The problem concerns a repeated game which has the same flavor as the online learning game of Section~\ref{sec:prelim}. The game is between a learner and an adversary. Additionally, there are $n$ experts. Each round $i$ in the game proceeds as follows:

\begin{enumerate}[(i)]
\item The experts present predictions $\hat{y}_i^{(1)},\ldots,\hat{y}_i^{(n)} \in \{0,1\}$.
\item The learner predicts a value $p_i \in [0,1]$.
\item The adversary reveals the true answer $y_i \in \{0,1\}$, and the learner suffers the loss $|y_i - p_i|$.
\end{enumerate}

The adversary must choose the answers so that at least one of the experts makes at most $k$ mistakes. That is, there must exist an expert $j$ such that $y_i \neq \hat{y}_i^{(j)}$ for at most $k$ many indices $i$. We call such an adversary \emph{$k$-consistent}.

The goal is to determine the optimal loss of the learner as a function of $n$ and $k$. We denote the optimal loss of the learner by $\M^\star(n,k)$, and the optimal loss when the learner is constrained to output predictions in $\{0,1\}$ by $\M^\star_D(n,k)$.

\smallskip

The game underlying prediction using expert advice is quite similar to the online learning game. In fact, we can relate the two.

Let $\cX_n = \{0,1\}^n$, and consider the hypothesis class $\cU_n$ on the domain $\cX_n$ consisting of the projection functions $h_i(x_1,\ldots,x_n) = x_i$.
We can simulate the game of prediction using expert advice by the online learning game as follows: whenever the experts predict $x_1,\ldots,x_n$, the adversary sends the instance $(x_1,\ldots,x_n)$. The adversary in the original game is $k$-consistent if and only if the sequence $(x_i,y_i)$ is $k$-realizable by $\cU_n$.

This simulation goes both ways, and so the two games are actually equivalent.
The upshot is that we can express $\M^\star(n,k)$ and $\M^\star_D(n,k)$ in terms of quantities we have already considered:
\[
 \M^\star(n,k) = \M^\star(\cU_n,k) = \RL_k(\cU_n) \text{ and } \M^\star_D(n,k) = \M^\star(\cU_n,k) = \LD_k(\cU_n).
\]


The equivalence above shows that $\cU_n$ is the ``hardest'' hypothesis class of size $n$, in the sense that it maximizes both $\M^\star(\mathcal{H},k)$ and $\M^\star_D(\mathcal{H},k)$ over all hypothesis classes $\mathcal{H}$ of size $n$. Indeed, $\M^\star(\mathcal{H},k)$ and $\M^\star_D(\mathcal{H},k)$ are equal to the optimal loss in the game of prediction using expert advice when the answers of the experts must belong to $ \{ (h_1(x),\ldots,h_n(x)) : x \in \cX \}$, where $\cH = \{h_1,\ldots,h_n\}$ has domain $\cX$.

\paragraph{Bounded horizon.} Prediction using expert advice is often considered when the number of rounds is bounded. Let $\M^\star(n,k,\expertrounds)$ be the optimal loss of the learner when the number of rounds is $\expertrounds$. 

Clearly $\M^\star(n,k,\expertrounds) \leq \M^\star(n,k)$. In view of Theorem~\ref{thm:characterization-rounds}, Proposition~\ref{pro:shallow-tree}, when extended to the $k$-realizable setting, shows that $\M^\star(n,k,\expertrounds) \geq \M^\star(n,k) - \epsilon$ already for $\expertrounds = 2\M^\star(n,k) + O(\sqrt{\M^\star(n,k)}\log(\M^\star(n,k)/\epsilon))$.
In contrast, since a learner can always guarantee a loss of at most $1/2$ per round by predicting $1/2$, we have $\M^\star(n,k,\expertrounds) \leq \expertrounds/2$, and so $\M^\star(n,k,\expertrounds) \geq \M^\star(n,k) - \epsilon$ requires $\expertrounds \geq 2\M^\star(n,k) - 2\epsilon$.

(The deterministic case is not interesting, since trivially $\M^\star_D(n,k,\expertrounds) = \min\{\expertrounds, \M^\star_D(n,k)\}$.)

\subsection{Optimal Mistake Bounds}

For every $n\geq1$ and $k\geq 0$, let \[
D(n,k)= \max\mleft\{d: d \leq \log n + \log \binom{d}{\leq k}\mright\}.
\]
The value of $D(n,k)$ plays a central role in the problem of prediction using expert advice: \cite{cesa1996line} showed that
$\M_D^\star(n,k) \leq D(n,k)$ using the \emph{Binomial Weights} learning rule, and complemented this with an asymptotically matching lower bound $\M_D^\star(n,k) \geq D(n,k) - o(D(n,k))$.
The lower bound is proved by constructing a $k$-covering code of size $n$ that simulates the experts. When $k$ is fixed and $n$ is large enough, it can be further improved to $\M_D^\star(n,k) \geq D(n,k) - 1$, as shown in~\cite{cesa1996line}.

The paper \cite{cesa1996line} leaves open the problem of
determining whether $\M^\star(n,k) \leq \frac{D(n,k)}{2} + c$ for some universal constant $c$.
\cite{cesa1997use} showed that $\M^\star(n,k) \leq \M^\star_D(n,k)/2 + o(\M^\star_D(n,k))$ whenever $k=o(\log n)$ or $k=\omega(\log n)$.\footnote{More precisely, they showed that $\M^\star(n,k) \leq k + \frac{\log n}{2} + \sqrt{k \ln n}$. Together with the bound $\M_D^\star(n,k) \geq 2k + \lfloor \log n \rfloor$ of \cite{littlestone1994weighted}, this implies that $\M^\star(n,k) \leq \M^\star_D(n,k)/2 + o(\M^\star_D(n,k))$ whenever $k=o(\log n)$ or $k=\omega(\log n)$.}
\cite{abernethy2006continuous} showed that for large enough $n$ (as a function of $k$),
$\M^\star(n,k) \leq \M^\star_D(n,k)/2 + O(1)$.
\cite{branzei2019online} showed that for $k=o(\log n)$, $\M^\star(n,k) \leq (1+o(1))\M^\star_D(n,k)/2$ even in the multiclass setting where the experts' predictions are chosen from some finite set $\{1, \dots, d\}$.
In this section, we remove any assumptions on $n,k$, proving the following theorem:

\begin{theorem}\label{thm:experts_upper_bound}
Let $n \geq 2$ and $k \geq 0$. Then
\[
 \M^{\star}(n,k) \leq D(n,k)/2 + O  \mleft(\sqrt{D(n,k)} \mright).
\]
%
\end{theorem}
The error term is tight for $n = 2$:
\begin{theorem} \label{thm:experts_lower_bound}
Let $k \geq 0$. Then 
\[
 \M^{\star}(2,k) = D(2,k)/2 + \Omega \mleft(\sqrt{D(2,k)} \mright).
\]
\end{theorem}



All of our bounds are attained using the randomized $k$-Littlestone dimension of $\cU_n$. Note that as a special case of Theorem~\ref{thm:k-realizable-littlestone-bound}, one can also derive the bounds $\M^{\star}(n,k) = k + \Theta\mleft(\sqrt{k \log n} + \log n \mright)$, using $\LD(\cU_n) = \lfloor \log n \rfloor$. The upper bound was proved by \cite{cesa1997use}.
We prove the upper \iffull\else bound \fi in Section~\ref{sec:Mnk-upper}, and the lower bound in Section~\ref{sec:Mnk-lower}.
\iffull
\else
In the full version of the paper, we use our techniques to determine $\M^\star(2,k)$ exactly.
\fi

\smallskip

All results we stated so far concern $n\geq 2$. The case $n=1$ is different, and much simpler:
\begin{theorem}\label{thm:single-expert-intro}
Let $k\geq 0$. Then
\[
\M^{\star}(1,k) = \M^{\star}_D(1,k) = D(1,k) = k.
\]
\end{theorem}
\begin{proof}
According to the definition, $D(1,k)$ is the maximum $d$ such that $2^d \leq \binom{d}{\leq k}$. Since $\binom{d}{\leq d} = 2^d$ whereas $\binom{d+1}{\leq d} < 2^{d+1}$, we see that $D(1,k) = k$.

\smallskip

The complete tree of depth $k$, labeled arbitrarily, is $k$-shattered by $\cU_1$. In contrast, a tree of depth $k+1$ cannot be $k$-shattered by $\cU_1$, since there exists a branch on which the unique hypothesis makes $k+1$ mistakes. Therefore $\M^{\star}_D(1,k)=k$.

\smallskip
\iffull
\else
For the randomized case,it is convenient to consider an optimal infinite tree $T_k$ such that $\M^\star(1,k) = \RL_k(\cU_1) = E_{T_k}/2$ (proof of existence of $T_k$ can be found in the full version of the paper \cite{FullVersion}).
\fi
Denote the unique hypothesis in $\cU_1$ by $h$. By possibly switching the order of children, we can assume that all vertices in $T_k$ are labeled by an instance $x$ such that $h(x) = 0$. We can then identify vertices of $T_k$ with binary strings.

Since $T_k$ is optimal, it contains all strings which contain at most $k$ many $1$s. A string is a leaf it it contains exactly $k$ many $1$s and it ends with $1$. The length of a random branch has the distribution of a sum of $k$ many $\operatorname{Geom}(1/2)$ random variables, and so $\M^\star(1,k) = E_{T_k}/2 = k$.
%
%
%
\end{proof}

In contrast, \cite{littlestone1994weighted} shows that $\M^\star_D(n,k) \geq 2k + \lfloor \log n \rfloor$ for $n \geq 2$, highlighting the difference between $n = 1$ and $n > 1$. This immediately implies the following corollary, which will be useful in the sequel:

\begin{corollary} \label{cor:D_lower_bound}
Let $n\geq 2$ and $k\geq 0$. Then $D(n,k) \geq 2k+1$.
\end{corollary}
\begin{proof}
Clearly $D(n,k) \geq D(2,k)$. Theorem~\ref{thm:experts_k_det_lower_bound} shows that $D(2,k) \geq M^{\star}_D(2,k)$, which is at least $2k+1$ by the result of~\cite{littlestone1994weighted}.
\end{proof}





\subsection{\iffull\else Proof \fi of the Upper \iffull\else Bound \fi on \texorpdfstring{$\M^\star(n,k)$}{M*(n,k)}} \label{sec:Mnk-upper}

We start by proving a probabilistic version of the \emph{sphere packing bound} for covering codes~\cite{cohen1997covering}.

\begin{lemma}
\label{lem:experts_k_realizing_class}
Let $\cH$ be a finite hypothesis class of size $n \ge 1$.
Let $t \geq k \geq 0$, and let $T$ be a tree whose minimum depth is at least $t$. 

Let $S = (x_1,y_1),\ldots,(x_t,y_t)$ be the \emph{random prefix of length $t$}, consisting of the first $t$ steps in a random branch of $T$. The probability that $S$ is $k$-realizable by $\cH$ is at most
\[
 n \binom{t}{\leq k} / 2^t.
\]
\end{lemma}
\begin{proof}
For each hypothesis $h \in \cH$ and set of indices $I \subseteq [t]$, the probability that $y_i \neq h(x_i)$ for all indices in $I$ and $y_i = h(x_i)$ for all indices outside of $I$ is $2^{-t}$.

The sequence $S$ is $k$-realizable by $\cH$ if the event above happens for some $h \in \cH$ and some $I$ of size at most $k$. Applying the union bound, we get that the probability is at most $n \binom{t}{\leq k} / 2^t$.
\end{proof}

As a warm-up, we use this lemma together with the $k$-Littlestone dimension to reprove the upper bound $\M^{\star}_D(n,k) \leq D(n,k)$, first proved in \cite{cesa1996line}.

\begin{theorem}\label{thm:experts_k_det_lower_bound}
Let $n\geq 1$ and  $k\geq 0$. Then $\M^{\star}_D(n,k) \leq D(n,k)$.
\end{theorem}

\begin{proof}
Since $\M^{\star}_D(n,k) = \LD_k(\cU_n)$, it suffices to bound $\LD_k(\cU_n)$.

Let $T$ be a tree satisfying $m_T = \LD_k(\cU_n)$ which is $k$-shattered by $\cU_n$. A random prefix of length $\LD_k(\cU_n)$ is $k$-realizable by $\cU_n$, and so $2^{\LD_k(\cU_n)} \leq n\binom{\LD_k(\cU_n)}{\leq k}$ by Lemma~\ref{lem:experts_k_realizing_class}. Taking the logarithm, we deduce that $\LD_k(\cU_n) \leq D(n,k)$ by the definition of $D(n,k)$.
\end{proof}

\iffull
\else
We now prove Theorem~\ref{thm:experts_upper_bound}.
\fi
The main tools are concentration of the random branch length in quasi-balanced trees (Lemma~\ref{lem:quasi_concentration}), and the following lemma.

\begin{lemma} \label{lem:k-realization-prob}
Let $\cH$ be a finite hypothesis class of size $n \ge 1$. Let $D = D(n,k)$, and let $T$ be a tree of minimum depth at least $(1+\epsilon) D$, where $0 < \epsilon < 1/3$. The probability that a random prefix of length $(1+\epsilon) D$ is $k$-realizable by $\cH$ is at most
\[
 2^{1 - \epsilon^2 D/9}.
\]

Furthermore, if $k \leq c D$ for some constant $c < 1/2$ then the probability is at most
\[
 2^{1 - c' \epsilon D},
\]
where $c' > 0$ is a constant depending only on $c$.
\end{lemma}

The proof of this lemma will require some elementary estimates on binomial coefficients, summarized in the following technical lemma. 

\begin{lemma} \label{lem:binomial-estimates}
Let $D \geq k \geq 1$ and $\epsilon > 0$. Then
\[
 \binom{(1+\epsilon) D}{\leq k} \leq 2^{\epsilon D \cdot \log (D/(D-k))} \cdot \binom{D}{\leq k}.
\]

If furthermore $k \leq D/2$ and $\epsilon \leq 1/3$ then
\[
 \binom{(1+\epsilon) D}{\leq k} \leq 2^{\epsilon D - \epsilon^2 k/3} \cdot \binom{D}{\leq k}.
\]
\end{lemma}

We prove this lemma in Subsection~\ref{sec:binomial-estimates}.

\begin{proof}[Proof of Lemma~\ref{lem:k-realization-prob}]
We start by observing that
\begin{equation} \label{eq:realizing_ability_low_sphere}
n \binom{D}{\leq k}/2^D \leq 2,  
\end{equation}
Indeed, the maximality of $D$ shows that
\[
 1 > n \binom{D+1}{\leq k}/2^{D+1} \geq \frac{1}{2} n \binom{D}{\leq k}/2^D,
\]
from which Eq.~\eqref{eq:realizing_ability_low_sphere} immediately follows. 

Denote by $p$ the probability we wish to bound. Lemma~\ref{lem:experts_k_realizing_class} shows that
\[
 p \leq n \binom{(1+\epsilon)D}{\leq k} / 2^{(1+\epsilon)D} =
 \frac{\binom{(1+\epsilon)D}{\leq k}}{\binom{D}{\leq k}} \cdot 2^{-\epsilon D} \cdot n \binom{D}{\leq k} / 2^D \leq
 2^{1-\epsilon D} \cdot \frac{\binom{(1+\epsilon)D}{\leq k}}{\binom{D}{\leq k}},
\]
using Eq.~\eqref{eq:realizing_ability_low_sphere}. It remains to estimate the ratio using Lemma~\ref{lem:binomial-estimates}.


We start by proving the ``furthermore'' part.
By assumption, we have $k \leq cD$.
Applying Lemma~\ref{lem:binomial-estimates}, we deduce that
\[
 p \leq 2^{1 - (1 - \log (D/(D - k))) \epsilon D}.
\]
Since
\[
 c' = 1 - \log \frac{D}{D - k} = 1 - \log \frac{1}{1 - k/D} \geq 1 - \log \frac{1}{1 - c} > 0,
\]
this completes the proof of the ``furthermore'' part.

In order to prove the main part of the lemma, we distinguish between two cases. If $k \leq D/3$ then the ``furthermore'' bound shows that
\[
 p \leq 2^{1 - c'\epsilon D},
\]
where $c' = \log (4/3)$. Since $\epsilon \leq 1/3$, we have $c' \epsilon \geq \epsilon^2/9$, completing the proof in this case.

Otherwise, $k \geq D/3$. In this case, noting that $k \leq D/2$ by Corollary~\ref{cor:D_lower_bound}, we apply the ``furthermore'' part of Lemma~\ref{lem:binomial-estimates} to obtain
\[
 p \leq 2^{1 - \epsilon^2 k/3} \leq 2^{1 - \epsilon^2 D/9}. \qedhere
\]
\end{proof}

We can now prove the upper \iffull\else bound \fi on $\M^\star(n,k)$. The idea is simple. Let $T$ be a tree which is $k$-shattered by $\cU_n$.
Using \iffull\else Corollary~\ref{cor:randomized-littlestone-quasi-balanced} \fi, we can assume that $T$ is quasi-balanced, and so the length of a random branch is concentrated around $E_T$.

This implies that $T$ realizes almost all sequences of size $(1-\epsilon) E_T$. These sequences are $k$-realized by $\cU_n$, and we obtain an upper bound on $E_T$ via Lemma~\ref{lem:k-realization-prob}.

\begin{proof}[Proof of Theorem~\ref{thm:experts_upper_bound}]
Since $\M(n,k) = \RL_k(\cU_n)$, we bound the latter.
\iffull
\else
It is convenient (although not necessary) to assume that there is an infinite tree $T$ which is $k$-shattered by $\cU_n$ and satisfies $E_T/2 = \RL_k(\cU_n)$\footnote{This assumption is proved to be true in the full version of this paper \cite{FullVersion}}. 
\fi
Furthermore, $T$ is monotone, and so Proposition~\ref{lem:quasi_concentration} applies to it (while the proposition is formulated for finite quasi-balanced trees, the proof actually directly uses monotonicity, and is valid for infinite trees).

In order to bound $E_T$, we will show that for small enough $\epsilon > 0$, the assumption $(1+\epsilon) D \leq (1-\epsilon) E_T$ leads to a contradiction.

Extend $T$ arbitrarily to a tree $T'$ of minimum depth $(1+\epsilon) D$, and let $S$ be a random prefix of $T'$ of length $(1+\epsilon) D$. If $S$ lies completely within $T$ then it is $k$-realizable by $\cU_n$, and so
\[
 \Pr[S \text{ lies within } T] \leq \Pr[S \text{ is $k$-realizable by } \cU_n].
\]

The probability that $S$ lies within $T$ is precisely the probability that a random branch of $T$ has length at least $(1+\epsilon) D$. Since we assume that $(1+\epsilon) D \leq (1-\epsilon) E_T$, this probability is at least $1 - e^{-\epsilon^2 E_T/4}$ by Proposition~\ref{lem:quasi_concentration}, and so at least $1 - e^{-\epsilon^2 D/4}$.

In contrast, the probability that $S$ is $k$-realizable by $\cU_n$ is at most $2^{1-\epsilon^2 D/9}$ by Lemma~\ref{lem:k-realization-prob}. Therefore
\[
 1 \leq e^{-\epsilon^2 D/4} + 2^{1-\epsilon^2 D/9}.
\]

Let $\epsilon = C/\sqrt{D}$. As $C \to \infty$, the right-hand side tends to $0$, and in particular, we obtain a contradiction for some constant $C > 0$.

It follows that $(1 + \epsilon) D > (1 - \epsilon) E_T$ for $\epsilon = C/\sqrt{D}$, and so
\[
 E_T < \frac{1 + \epsilon}{1 - \epsilon} D = (1 + O(1/\sqrt{D})) D = D + O(\sqrt{D}). \qedhere
\]
\end{proof}

\subsubsection{Proof of Technical Estimate}
\label{sec:binomial-estimates}

\iffull
\else
In this section we complete the proof of Theorem~\ref{thm:experts_upper_bound} by proving Lemma~\ref{lem:binomial-estimates}.
\fi

We start with estimates on the ratio of individual binomial coefficients.

\begin{lemma} \label{lem:binomial-estimates-aux}
Let $D \geq \ell \geq 1$ and $\epsilon > 0$. Then
\[
 \binom{(1+\epsilon) D}{\ell} \leq 2^{\epsilon D \cdot \log (D/(D-\ell))} \cdot \binom{D}{\ell}.
\]

If furthermore $\ell \leq D/2$ and $\epsilon \leq 1/3$ then
\[
 \binom{(1+\epsilon) D}{\ell} \leq 2^{\epsilon D \cdot \log (D/(D-\ell)) - \epsilon^2 \ell/3} \cdot \binom{D}{\ell}.
\]
\end{lemma}
\begin{proof}
We can calculate the ratio between the binomials explicitly:
\[
 R_\ell := \left. \binom{(1+\epsilon) D}{\ell} \middle/ \binom{D}{\ell} \right. =
 \prod_{r=0}^{\ell-1} \frac{(1+\epsilon)D - r}{D - r} = \prod_{r = 0}^{\ell-1} \left(1 + \frac{\epsilon D}{D - r}\right).
\]
Applying the well-known estimate $\ln (1 + x) \leq x$, we obtain
\[
 \ln R_\ell \leq \sum_{r = 0}^{\ell - 1} \frac{\epsilon D}{D - r} \leq \epsilon D \cdot \int_{D-\ell}^D \frac{dx}{x} = \epsilon D \cdot \ln \frac{D}{D-\ell},
\]
and so
\[
 R_\ell \leq
 2^{\epsilon D \cdot \log (D/(D-\ell))}.
\]

\smallskip

Now suppose that $\ell \leq D/2$ and $\epsilon \leq 1/3$. For $r \in \{0,\ldots,\ell-1\}$ we have
\[
 \frac{\epsilon D}{D - r} \leq
 \frac{\epsilon D}{D - \ell} =
 \frac{\epsilon}{1 - \ell/D} \leq 2\epsilon \leq 2/3.
\]
Since $1 + x \leq e^{x - x^2/3}$ for $x \leq 0.787$, we can improve the estimate on $R_\ell$:
\[
 \ln R_\ell \leq \epsilon D \cdot \ln \frac{D}{D-\ell} - \frac{1}{3} \sum_{r = 0}^{\ell - 1} \frac{\epsilon^2 D^2}{(D - r)^2} \leq \epsilon D \cdot \ln \frac{D}{D - \ell} - \frac{1}{3} \epsilon^2 \ell. \qedhere
\]
\end{proof}

We can now prove Lemma~\ref{lem:binomial-estimates}.

\begin{proof}[Proof of Lemma~\ref{lem:binomial-estimates}]
The ratio between $\binom{(1+\epsilon) D}{\leq k}$ and $\binom{D}{\leq k}$ is clearly at most $\max(R_0,\ldots,R_k)$, where $R_\ell$ is the ratio between the binomials in Lemma~\ref{lem:binomial-estimates-aux}.

If we only assume that $D \geq \ell \geq 1$ and $\epsilon > 0$, then Lemma~\ref{lem:binomial-estimates-aux} states that
\[
 \log R_\ell \leq \epsilon D \cdot \log \frac{D}{D - \ell},
\]
which is clearly monotone increasing in $\ell$. Therefore
\[
 \log \max(R_0,\ldots,R_k) \leq \epsilon D \cdot \log \frac{D}{D - k}.
\]
If we furthermore assume that $k \leq D/2$ and $\epsilon \leq 1/3$, then Lemma~\ref{lem:binomial-estimates-aux} states that
\[
 \log R_\ell \leq \epsilon D \cdot \log \frac{D}{D - \ell} - \frac{1}{3} \epsilon^2 \ell.
\]
The derivative of the upper bound with respect to $\ell$ is
\[
 \frac{\epsilon D}{D - \ell} - \frac{1}{3} \epsilon^2 \geq \epsilon - \frac{1}{3} \epsilon^2 > 0,
\]
since $\epsilon \leq 1/3$. Therefore the upper bound is maximized at $\ell = k$, and we conclude that
\[
 \log \max(R_0,\ldots,R_k) \leq \epsilon D \cdot \log \frac{D}{D - k} - \frac{1}{3} \epsilon^2 k.
\]
Since $k \leq D/2$, we can further estimate
\[
 \log \frac{D}{D - k} = \log \frac{1}{1 - k/D} \leq \log 2 = 1. \qedhere
\]
\end{proof}

\subsection{Lower bounding \texorpdfstring{$\M^\star(2,k)$}{M*(2,k)}} \label{sec:Mnk-lower}
We prove Theorem~\ref{thm:experts_lower_bound} by applying the lower bound of Theorem~\ref{thm:k-realizable-littlestone-bound}. \iffull\else In the full version of this paper \cite{FullVersion}, we also show how our techniques can be used to determine the exact value of $\M^\star(2,k)$. Concretely, we are able to identify the optimal shattered tree and to compute its expected branch length. \fi

\begin{proof}[Proof of Theorem~\ref{thm:experts_lower_bound}]
Since $\LD(\cU_2) = 1$, Theorem~\ref{thm:k-realizable-littlestone-bound} implies that
\[
 \M^\star(2,k) = \RL_k(\cU_2) = k + \Omega \mleft(\sqrt{k} \mright).
\]
On the other hand, it is easy to see that $D(2,k) = 2k+1$. Indeed, if $d \geq 2k+1$ then
\[
 \log 2 + \log \binom{d}{\leq k} \leq \log 2 + \log 2^{d-1} = d,
\]
with equality if and only if $d = 2k+1$. Therefore $D(2,k) = 2k+1$.
\end{proof}

\section{Open Questions}

\iffull
\else
Our work naturally raises many directions for future research (more can be found in the full version paper \cite{FullVersion}).
\paragraph{Adaptive algorithms.}
Algorithm $\WRSOA$ gives the optimal mistake bound, but requires knowledge of $k$.
Theorem~\ref{thm:adaptive-algorithm} gives an algorithm which doesn't require knowledge of $k$, and has a regret bound of $\tilde{O}(\sqrt{\M^\star(\cH,k) \cdot \log k})$ (this is the loss beyond $\M^\star(\cH,k)$).
What is the optimal regret bound?

\paragraph{Quantitative bounds.}
The first part in Theorem~\ref{thm:expertIntro} asserts
that $\M^\star(n,k)\leq \frac{1}{2}D(n,k) + O \mleft( \sqrt{D(n,k)} \mright)$.
It will be interesting to get quantitative bounds
on the second-order term in terms of $\M^\star_D(n,k)$.
By the second part of Theorem~\ref{thm:expertIntro} we know that in some cases ($n=2$)
it is $\Omega \mleft(\sqrt{\M^\star_D(n,k)} \mright)$. Does an upper bound of 
\[\M^\star(n,k)\leq \frac{1}{2}\M^\star_D(n,k) + O\mleft(\sqrt{\M^\star_D(n,k)}\mright)\]
hold for all $n,k$?

In addition, it will be interesting to find explicit bounds on $\M^\star(n,k), \M^\star_D(n,k)$; by Theorem~\ref{thm:k-realizable-littlestone-bound-Intro}, we know that when $k \gg \log n$ then\footnote{The upper bound is also given in \cite{cesa1997use}.}
    \begin{align*}
    \M^\star(n,k) &= k + \Theta\mleft(\sqrt{k\log n}\mright).
    \end{align*}
    How about for other values of $n,k$? 
    Br\^anzei and Peres~\cite{branzei2019online} have the state-of-the-art bounds in the regime $k \ll \log n$,
    but we are not aware of any results in other regimes, e.g.\ when $k=\Theta(\log n)$. As a matter of fact, to the best of our knowledge, even the leading asymptotic terms in the regime $k=\Theta(\log n)$ were unknown prior to this work (we included them in the full version of this work \cite{FullVersion}).

\paragraph{Proper predictions and repeated game playing.}
Consider the prediction using the expert advice problem, when the learner is restricted to predict with a convex combination of the experts. That is, at the beginning of each round (before seeing the advice of the $n$ experts), the learner picks a convex combination of the experts and predicts accordingly. What is the optimal expected number of mistakes in this game?\footnote{We answer this question only for the case $k=0$ in the full version of this paper \cite{FullVersion}.}
The optimal algorithm for the hedge setting \cite{freund1997decision} was identified in \cite{abernethy2008optimal2}.

We comment that this game can also be presented in the language of game theory:
assume a repeated zero-sum game with $0/1$ values, where each round is played as follows: player~(i)  chooses a (mixed) strategy and reveals it to player~(ii), who then replies with a strategy of his own.  
What is the optimal accumulated payoff that player~(i) can guarantee provided that she has $n$ strategies and that the sequence of strategies chosen by player~(ii) is such that player~(i) has a pure strategy that loses to at most $k$ of them? Proper predictions in the prediction with expert advice setting are equivalent to mixed strategies here.

\fi

\section*{Acknowledgments}
Shay Moran is a Robert J.\ Shillman Fellow; he acknowledges support by ISF grant 1225/20, by BSF grant 2018385, by an Azrieli Faculty Fellowship, by Israel PBC-VATAT, by the Technion Center for Machine Learning and Intelligent Systems (MLIS), and by the the European Union (ERC, GENERALIZATION, 101039692). Views and opinions expressed are however those of the author(s) only and do not necessarily reflect those of the European Union or the European Research Council Executive Agency. Neither the European Union nor the granting authority can be held responsible for them.

\bibliographystyle{alphaurl}
\bibliography{bib.bib}

\appendix

\end{document}